\documentclass[10pt]{article}
\usepackage[top=2cm, bottom=2cm, left=2cm, right=2cm]{geometry}
\usepackage{amsthm,amsmath,amsfonts,amssymb}
\usepackage[sort]{natbib}
\newtheorem{remark}{Remark}[section]
\newtheorem{definition}{Definition}[section]
\newtheorem{lemma}{Lemma}[section]
\newtheorem{proposition}{Proposition}[section]
\newtheorem{theorem}{Theorem}[section]
\newtheorem{informal}{Informal theorem}[section]

\usepackage{mathtools}
\usepackage{mathrsfs}
\usepackage{stmaryrd}
\usepackage{bm}
\usepackage{bbm}
\usepackage{enumitem}
\usepackage{booktabs}
\usepackage{array}
\usepackage{longtable}
\usepackage{adjustbox}
\usepackage{amsfonts} 

\usepackage{algorithm}
\usepackage{algpseudocode}

\usepackage{graphicx}
\usepackage{subcaption}
\usepackage{tikz}
\usepackage{multirow}
\usepackage{wrapfig}

\usepackage{hyperref}
\definecolor{graphred}{HTML}{B20000}
\definecolor{graphgreen}{HTML}{82B366}
\definecolor{graphblue}{HTML}{6D7BBF}
\definecolor{graphpurple}{HTML}{8E44AD}
\definecolor{lasallegreen}{rgb}{0.03, 0.47, 0.19}
\hypersetup{
  colorlinks=true,
  urlcolor=lasallegreen,
  linkcolor=lasallegreen,
  citecolor=lasallegreen,
  bookmarksdepth=paragraph
}
\usepackage[nameinlink]{cleveref}

\crefname{equation}{}{}
\crefname{section}{section}{sections}
\crefname{figure}{figure}{figures}
\crefname{table}{table}{tables}
\crefname{example}{example}{examples}
\crefname{proposition}{proposition}{propositions}
\Crefname{section}{Section}{Sections}
\Crefname{figure}{Figure}{Figures}
\Crefname{table}{Table}{Tables}
\Crefname{definition}{Definition}{Definitions}
\Crefname{theorem}{Theorem}{Theorems}
\Crefname{remark}{Remark}{Remarks}
\Crefname{example}{Example}{Examples}
\Crefname{proposition}{Proposition}{Propositions}

\newcommand{\eqdef}{\mathrel{\stackrel{\scriptscriptstyle\mathrm{def}}{=}}}

\newcommand{\R}{\mathbb{R}}
\newcommand{\N}{\mathbb{N}}

\newcommand{\m}{\mathtt{m}}
\newcommand{\h}{\mathtt{h}}

\renewcommand{\hat}{\widehat}

\renewcommand{\tilde}{\widetilde}

\renewcommand{\epsilon}{\varepsilon}

\newcommand{\op}{\mathrm{op}}

\usepackage{authblk}

\title{Statistical Guarantees for Reasoning Probes on Looped Boolean Circuits}

\author[1,2]{Anastasis Kratsios}
\author[3]{Giulia Livieri}
\author[4]{A.~Martina Neuman%
\thanks{Corresponding author. Email: \texttt{neumana53@univie.ac.at}}}

\affil[1]{Department of Mathematics, McMaster University, Canada}
\affil[2]{Vector Institute, Canada}
\affil[3]{The London School of Economics and Political Science}
\affil[4]{University of Vienna, Faculty of Mathematics}

\affil[ ]{\vspace{1ex}
\texttt{kratsioa@mcmaster.ca} \quad
\texttt{g.livieri@lse.ac.uk} \quad
\texttt{neumana53@univie.ac.at}
}

\date{}

\usepackage[most]{tcolorbox}
\usepackage{etoolbox}
\usepackage{xcolor}

\definecolor{faintgray}{RGB}{245,245,245}
\definecolor{faintborder}{RGB}{230,230,230}
\definecolor{lightblack}{gray}{0.4}

\usepackage[most]{tcolorbox}
\usepackage{etoolbox} 
\usepackage{xcolor}
\definecolor{faintgray}{RGB}{245,245,245}     
\definecolor{faintborder}{RGB}{230,230,230}   
\definecolor{lightblack}{gray}{0.4}           

\begin{document}

\maketitle

\begin{abstract}
We study the statistical behavior of reasoning probes in a stylized model of iterative computation inspired by neural algorithmic reasoning. The underlying computation is given by a looped Boolean circuit whose graph is a perfect $\nu$-ary tree ($\nu\ge 2$), with outputs recursively fed back as inputs across computation rounds. A probe observes a sampled subset of internal nodes and seeks to infer the latent operation at each node, represented as a probability distribution over a finite set of admissible Boolean gates. This partial observability induces a transductive generalization problem on a structured computation graph.
We show that when the probe is parameterized by a graph convolutional network and queries $N$ nodes, the worst-case generalization error decays at the optimal rate $\mathcal{O}(\sqrt{\log(2/\delta)}/\sqrt{N})$ with probability at least $1-\delta$. Our analysis combines metric embedding techniques with tools from optimal transport. A key insight is that this rate is achievable independently of the size of the computation graph, enabled by a low-distortion one-dimensional snowflake embedding of the induced graph metric. These results highlight a geometric mechanism underlying statistical efficiency in probing structured, iterative computations.
\end{abstract}

\section{Introduction}

Recent advances in machine learning have led to systems that exhibit increasingly sophisticated forms of \emph{multi-step reasoning}, including chain-of-thought prompting, tool use, and iterative self-refinement. Prominent examples include agentic systems~\citep{shinn2023reflexion,wang2024voyager}, tool-augmented models~\citep{lewis2020retrieval,asai2024selfrag}, and iterative self-improvement strategies~\citep{madaan2023self,wang2023selfconsistency,zelikman2022star,chen2025magicore,song2025thinking}. Concurrently, large pretrained language models (such as the GPT family~\citep{brown2020language,achiam2023gpt} and Gemini~\citep{gemini2023}) have demonstrated strong performance on widely used \emph{proxies} for reasoning. 
As such components are embedded within increasingly complex machine learning pipelines, there is growing interest in understanding and explaining how these systems produce their outputs~\citep{bostrom2014superintelligence,bengio2024managing}.
Despite strong empirical progress, a fundamental question remains poorly understood:
\begin{center}
\emph{``To what extent can we recover the internal computational structure of a reasoning process from partial observations?''}
\end{center}
This question arises naturally in interpretability~\cite{DoshiVelezKim2017Interpretability,Lipton2018Mythos,OlahEtAl2020ZoomIn,ElhageEtAl2021TransformerCircuits} and auditing~\cite{MitchellEtAl2019ModelCards,GebruEtAl2021Datasheets,BuolamwiniGebru2018GenderShades,RajiEtAl2020AccountabilityGap}: one may wish to identify latent operations executed during reasoning, or determine whether a model follows a consistent computational strategy across inputs.  In practice, however, internal computation is only partially observed, often with noise or uncertainty.

\paragraph{The neural algorithmic reasoning lens.}
We approach this question through the lens of neural algorithmic reasoning (NAR), popularized by~\cite{velivckovic2021neural}, which interprets the reasoning capabilities of neural networks as arising from their ability to implement discrete computational procedures. Under this perspective, intermediate computations are expressed in an interpretable form, often modeled as Boolean circuits~\cite{jukna2012boolean} composed of elementary operations such as logical gates $\{\operatorname{AND},\operatorname{OR},\operatorname{NOT}\}$ or expressive threshold functions.
This viewpoint traces back to the origins of neural network theory~\cite{mcculloch1943logical}, was developed in classical studies of threshold and circuit models~\cite{muroga1971threshold,maass1991sigmoid,parberry1994circuit}, and has reemerged in modern form as neural algorithmic reasoning~\cite{velivckovic2021neural,velickovic2020neural,velickovic2022clrs,ibarz2022generalist,selsam2018learning}.

\paragraph{Why study Boolean circuits directly and not neural networks emulating them?}
In this line of research, neural networks are viewed as Boolean algorithms, and reasoning is formalized through their ability to \textit{compute} Boolean functions, with intermediate ``chain-of-thought'' steps represented in an interpretable Boolean language.  We therefore take the emulated Boolean circuit as our starting point, allowing general Boolean gates of bounded arity $\m\ge 2$ so as not to restrict the analysis to any particular neural architecture or Boolean language, e.g.\ binary circuits over fixed gates.

\paragraph{Probing reasoning as a transductive learning problem.} A common approach to studying internal computation is through \emph{probing methods}, which train auxiliary models to predict properties of hidden states. 
These methods form a standard toolkit in the interpretability literature, ranging from low-complexity probes~\citep{Alain2016linearprobes,DBLPjournals/corr/abs-1805-01070,bau2017network} to more expressive models~\citep{craven1996trepan,contreras2022dexire,Olah2018building,Olah2020zoom,dai-etal-2022-knowledge}. Typically, such probes are used to predict properties of internal representations, such as intermediate computations, at selected locations within a model.
While widely used in practice, the statistical behavior of such probes—especially in structured, iterative reasoning settings—remains largely unexplored.
In this work, we formulate probing as a \emph{statistical learning problem under partial observability}. We consider a setting in which a reasoning process unfolds over a structured computational graph, and each node corresponds to an intermediate computation. The underlying computation at each node is governed by a latent operation (e.g., a Boolean gate), which we aim to infer.
Rather than assuming full access to the computation, the probe observes only a \emph{sampled subset of nodes}, together with noisy signals reflecting uncertainty about the underlying operation. This naturally leads to a \emph{transductive learning} setting, in which the central question is whether performance on observed nodes generalizes to unobserved nodes of the same computation graph.

\paragraph{A stylized model of iterative reasoning.}
To make this problem tractable, we introduce a stylized model of computation based on looped Boolean circuits. Computation proceeds over a \emph{tree-structured graph} (specifically, a perfect $\nu$-ary tree), in which each internal node implements an operation drawn from a finite set (e.g., a Boolean gate with fan-in $\nu$ and fan-out $1$). The output of the circuit is recursively fed back as input, inducing a strongly connected directed graph (digraph). This captures key features of modern reasoning systems, including:
\begin{itemize}
    \item[(i)] \emph{iterative refinement} via feedback,
    \item[(ii)] \emph{structured computation} over intermediate states,
    \item[(iii)] and \emph{global information flow} across multiple steps.
\end{itemize}
We consider probes operating under \emph{intrinsic uncertainty}, where observations at each node take the form of noisy distributions over possible operations. The probe's task is to infer these node-level properties across the graph from partial observations.
To model probes that respect the computational structure, we consider hypothesis classes parameterized by graph neural networks, specifically \emph{graph convolutional networks} (GCNs). These architectures implement message passing and are widely used in neural algorithmic reasoning to approximate discrete algorithms on graphs.

\subsection{Problem formulation}

We now formalize the probing setup and the computational model underlying our analysis, in stages.

\paragraph{Model of reasoning probe and probe uncertainty.} 
Let $\m\in\mathbb{N}$ with $\m\geq 2$.
We model probe outputs as elements of the relative interior $\Delta_{\m}^{\circ}$ of the $\m$-simplex $\Delta_{\m}$---for a detailed definition, see Section~\ref{s:Prelims}---which we later interpret as encoding uncertainty over candidate gates.
Each element in $\Delta_{\m}$ represents a probability distribution over the $\m$ candidate gates, drawn from a fixed finite set.
Degenerate distributions correspond to boundary points of the simplex, which lie in $\Delta_{\m}\setminus \Delta_{\m}^{\circ}$.
We equip $\Delta_{\m}^{\circ}$ with the vector space structure induced by the Aitchison geometry; see Section~\ref{s:Prelims} and also~\citep{pawlowsky2006compositional, pal2020multiplicative, erb2021information}. Under this geometry, the boundary of $\Delta_{\m}$ lies at infinite distance and is therefore unattainable. This precludes degenerate distributions and enforces intrinsic uncertainty in probe outputs.

\paragraph{Stylized model of iterative Boolean reasoning.}
We adopt a classical view of reasoning as the execution of algorithms or circuits, which aligns naturally with the neural algorithmic reasoning paradigm. Under this lens, many modern neural architectures can be interpreted as simulating discrete computational processes, leading to strong expressivity results. Representative examples include the Turing machine viewpoint for recurrent or looped networks~\citep{siegelmann2012neural,perez2018on,chung2021turing,bournez2025universal}, and Boolean and circuit-complexity perspectives for feedforward architectures~\citep{maass1991sigmoidvsbool,merrill2022saturated,merrill2023parallelism,chiang2025transformers}. 

We model reasoning as a Boolean circuit computing a $B$-bit Boolean function $f:\{0,1\}^B\to \{0,1\}$, represented as a \emph{directed acyclic graph} (DAG) of Boolean gates. 
When the circuit is a perfect $\nu$-ary tree of height $\h$, then $B=\nu^{\h}$, and each internal computation node has fan-in $\nu$ and fan-out $1$ Boolean gate. 
As a point of reference, when $\nu=2$, every Boolean function admits such a DAG representation using binary gates (e.g., $\mathrm{AND}$, $\mathrm{OR}$) with input literals~\citep[Lemma 15]{LiKratsiosGhoukasianZvigelsky_CertifiableReasoningUniversal}.

To model \emph{iterative} reasoning, we allow the circuit output to be recursively fed back as input over multiple steps, in analogy with chain-of-thought and iterative refinement~\citep{wei2022chain,kojima2022large,saunshi2025reasoning}; see Figure~\ref{fig:loopedreasoning_vs_graph}. This produces a sequence of intermediate computational states, which we view as being written to a memory \emph{tape} initialized by a $B$-bit prompt.

Formally, let $\mathtt{B}_{\nu}=(V_{\nu},E_{\nu})$ be a \emph{perfect} (\emph{full} and \emph{complete}) rooted $\nu$-ary tree with edges oriented toward the root $\mathtt{r}$, and let $\mathtt{T}\eqdef\{T_0,T_1,\dots\}$ denote a (possibly infinite) tape. At time $t\in\mathbb{N}{\ge 0}$, we consider copies $\mathtt{B}_{\nu}^{(t)}$ and $\mathtt{T}^{(t)}$, whose evolution is governed by three operations:
\begin{equation} \label{update}
    \begin{alignedat}{3}
        &\mathtt{read} &&: \mathtt{T}^{(t)} \to \mathtt{B}_{\nu}^{(t)} && \\
        &\mathtt{shift} &&: \mathtt{T}^{(t)} \to \mathtt{T}^{(t+1)} &&\quad\text{ where }\quad T_i^{(t)}\mapsto T_{i+1}^{(t+1)} \\
        &\mathtt{write} &&: \mathtt{B}_{\nu}^{(t)} \to \mathtt{T}^{(t+1)} &&\quad\text{ where }\quad \mathtt{r}^{(t)}\mapsto T_0^{(t+1)}.
    \end{alignedat}
\end{equation}
That is, the tape is read into the computation tree, shifted forward, and updated with the resulting current output. Informally, this defines a \emph{causal} evolution $(\mathtt{B}_{\nu}^{(t)},\mathtt{T}^{(t)}) \mapsto (\mathtt{B}_{\nu}^{(t+1)},\mathtt{T}^{(t+1)})$. 
See Figure~\ref{fig:loopedreasoning_vs_graph}(a).
Moreover, the $\mathtt{read}$ operation is encoded by an underlying DAG $G^{(t)}=(V^{(t)},E^{(t)})$:
\begin{equation} \label{dynamics}
    V^{(t)} = \{T_i^{(t)}: i\in [\nu^{\h}]_0\} \cup V_{\nu}^{(t)} \quad\text{ and }\quad
    E^{(t)} = \{(T_i^{(t)}, v_{i+1}^{(t)}): i\in [\nu^{\h}]_0 \} \cup E_{\nu}^{(t)}.
\end{equation}
Here, $[\nu^{\h}]_0$ denotes the first $\nu^{\h}$ integers starting with $0$, and $v_i^{(t)}$ denotes the time-$t$ copy of the base (input) tree node $v_i$, taking Boolean values.
Thus, from \eqref{dynamics} at any given time $t$, only the first $\nu^{\h}$ tape cells interact with the computation core. 
Each internal node $v\in V_{\nu}$ is assigned a Boolean gate $\mathfrak{g}_v:\{0,1\}^{\nu}\to \{0,1\}$ from a fixed finite set, yielding a $\nu^{\h}$-bit Boolean function at each time step.
Finally, combining intra-time edges with the inter-time edges induced by $\mathtt{shift}$ and $\mathtt{write}$ gives a directed graph $\mathcal{G}^{\mathrm{time}}$ capturing global information flow. 
Its time-quotient $\mathcal{G}^{\mathrm{time}}/\mathbb{N}{\ge 0}$ is a strongly connected digraph $G{\mathrm{sc}}=(V_{\mathrm{sc}},E_{\mathrm{sc}})$ given by
\begin{equation} \label{awesomegraph}
    \begin{split}
    V_\mathrm{sc} &\eqdef V_{\nu} \cup \{T_i\}_{i\in [\nu^{\h}]_0} \\
    E_\mathrm{sc} &\eqdef E_{\nu} \cup \{(\mathtt{r},T_0)\} \cup\{(T_i,T_{i+1})\}_{i\in [\nu^{\h}]_0} \cup \{(T_i,v_{i+1})\}_{i\in [\nu^{\h}]_0}.
    \end{split}
\end{equation}
See Figure~\ref{fig:loopedreasoning_vs_graph}(b). In this way, the looped execution embeds the feedforward Boolean circuit into a strongly connected directed graph.

\begin{figure}[t]
    \centering
    \begin{minipage}[t]{0.49\linewidth}
        \centering
        \includegraphics[width=\linewidth]{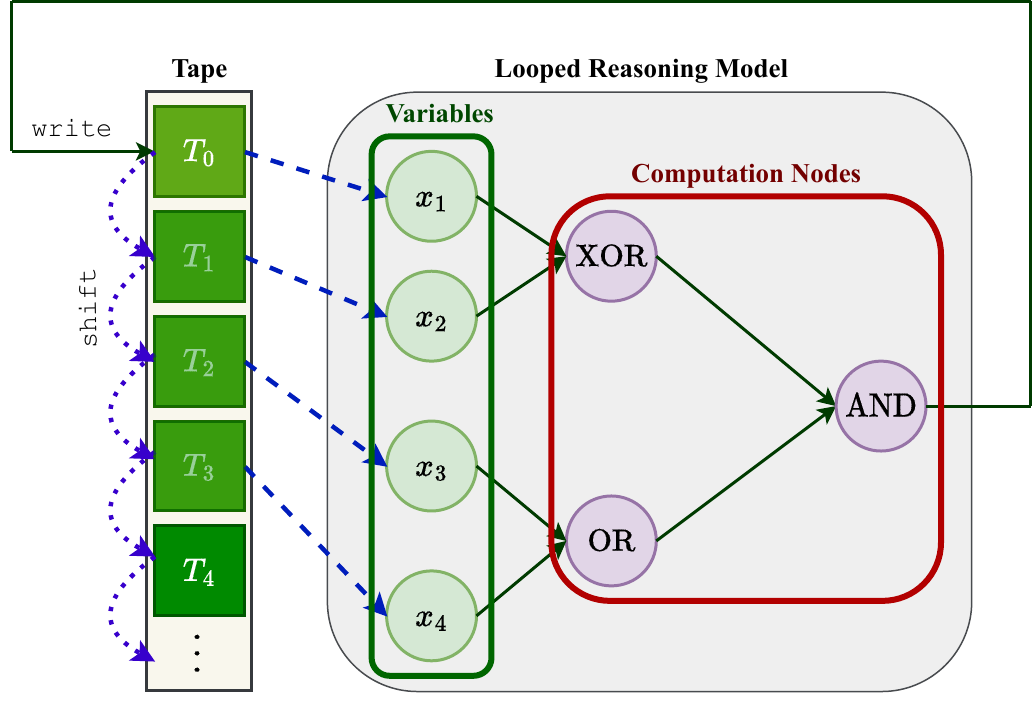}
        \hfill\\
        \vspace{.2em}
        \emph{(a) Looped reasoning model}
        \label{fig:loopedreasoning}
    \end{minipage}\hfill
    \begin{minipage}[t]{0.49\linewidth}
    \centering\includegraphics[width=.7\linewidth]{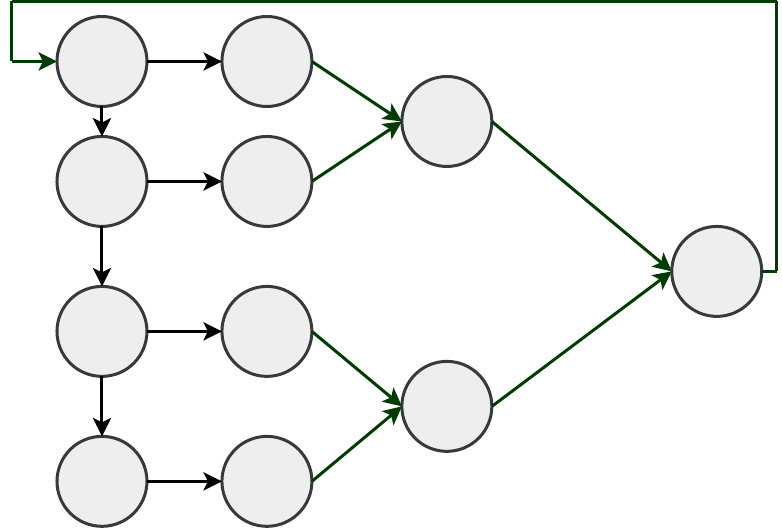}
        \hfill\\
        \vspace{.2em}
        \emph{(b) \emph{Strongly connected digraph} $\mathcal{G}^{\mathrm{time}}/\mathbb{N}_{\ge 0}$}
    \label{fig:loopedreasoning_graph}
    \end{minipage}
    \caption{A looped reasoning model (a) and its induced time-quotient, strongly connected digraph (b).
    In the looped reasoning model in (a), at time step $t\in\N$, the bits stored at the tape positions $T_0^{(t)},T_1^{(t)},T_2^{(t)},T_3^{(t)}$ are fed as input variables at the base nodes of a perfect binary rooted tree ({\color{graphgreen}{green nodes}}). The computation nodes ({\color{graphred}{red-framed box}}) apply binary gates; the output at root $\mathtt{r}^{(t)}$ ({\color{graphpurple}{purple node}} with AND operation) is fed back to tape $T_0^{(t+1)}$.
    }
    \label{fig:loopedreasoning_vs_graph}
\end{figure}

\paragraph{Model of reasoning probes as GCNs.} All probes we consider are Lipschitz with respect to the hitting-probability metric on the computation graph (see Proposition~\ref{prop:est}). From a learning-theoretic viewpoint, this Lipschitz condition on probe outputs defines a function class that varies smoothly across the graph: nearby nodes have similar gate distributions. This induces an approximately homophilic structure, in which local neighborhoods exhibit limited variation. 
This property naturally motivates the use of GCN-based probes, which are well suited to such settings~\citep{ma2022is}. 
Accordingly, we consider GCNs as concrete, computationally tractable realizations of this function class, analogous to how neural networks approximate classical function classes in Euclidean settings. A formal definition is given in Definition~\ref{defn:GGCN} (Section~\ref{s:Prelims}).

\begin{remark}[Scope and modeling choices]
Our goal is not to model the full complexity of learned neural representations, but to isolate a core statistical question that arises in probing: how well can one infer latent computational structure from partial, noisy observations of intermediate states? To this end, we have adopted a stylized model based on looped Boolean circuits, which captures key structural features of many reasoning systems---namely, iterative computation, compositional structure, and global information flow---while remaining amenable to analysis. 

This abstraction allows us to formulate probing as a transductive learning problem on a fixed computation graph, where observations are available only on a subset of nodes. In this setting, generalization corresponds to predicting unobserved parts of the same computation, a task that is central to understanding internal structure in large models where full access is unavailable. While simplified, this formulation isolates the statistical and geometric challenges inherent in probing and enables precise guarantees that would be difficult to obtain in more realistic but less tractable settings.
\end{remark}

\subsection{Overview and informal theorem}

We summarize our main contributions and state an informal version of our main result.

\paragraph{Main contribution.}
We establish statistical generalization guarantees for GCN-parameterized reasoning probes (Theorem~\ref{thrm:main_result}). 
In particular, our results show that, even for large, expanding, computation graphs, GCN-parameterized reasoning probes can achieve statistically optimal $\mathcal{O}(1/\sqrt{N})$ rates independent of graph cardinality. 
An informal version of our main theorem is as follows.

\begin{informal}
Consider a fixed computation graph with $k$ nodes, and suppose we observe $N$ noisy probe outputs at randomly sampled nodes. Then, for GCN-based probe hypothesis class $\mathcal{H}$, the (worst-case) generalization gap between empirical risk $\hat{\mathcal{R}}_N$ and population risk $\mathcal{R}$ satisfies
\[
\sup_{h \in \mathcal{H}} \big|\mathcal{R}(h) - \hat{\mathcal{R}}_N(h)\big| \;\lesssim\; \frac{1}{\sqrt{N}} + \frac{\sqrt{\log(2/\delta)}}{\sqrt{N}}
\]
with probability at least $1-\delta$ and a majorant constant that can be chosen independent of $k$ (see Section~\ref{s:Main__ss:Discussion}).
\end{informal}

\paragraph{Secondary (technical) contributions.}
At a high level, the difficulty arises from the interaction of three factors: (i) partial observability of the computation graph, (ii) structured dependencies induced by iterative (looped) computation, and (iii) the need to generalize across a fixed graph rather than over independent samples. Standard i.i.d.\ learning tools do not directly apply in this regime, as observations are both dependent and spatially constrained.

Our analysis therefore relies on two technical developments of independent interest.
First, we obtain a one-dimensional snowflake embedding for finite metric spaces. 
Second, we derive Lipschitz estimates for GCNs on digraphs; to the best of our knowledge, these are the first such results.
Together, the embedding and Lipschitz analysis identify metric geometry as the key factor governing statistical efficiency under partial observability. For a brief proof sketch and further explanation, see Section~\ref{s:proofstrategy}.

\paragraph{Organization.} Section~\ref{s:Prelims} introduces all notation and background needed to formalize our main results.
Section~\ref{s:Main} presents our main result (Theorem~\ref{thrm:main_result}). 
Technical discussions on our proof technique are provided in Section~\ref{s:Main__ss:Discussion}.
All proofs and additional theoretical background are deferred to the appendix.


\section{Preliminaries}
\label{s:Prelims}

We write $\mathbb{N}$ to denote the set of natural numbers, $\mathbb{N}_{\geq 0}$ to denote the set of non-negative integers, and $\mathbb{R}_{\geq 0}$ to denote the set of non-negative reals.
Let $k\in\N$. We also write $[k]\eqdef \{1,2,\dots,k\}$ and $[k]_0\eqdef \{0,1,\dots,k-1\}$. 
For a finite set $V$, we denote by $\#V$ its cardinality. 

We use boldface letters to denote Euclidean vectors, e.g. $\mathbf{p}$, $\mathbf{x}$, $\mathbf{Y}$, and write $\|\mathbf{x}\|_{\infty}$ for the $\ell^{\infty}$-norm of $\mathbf{x}$.
For a linear operator $W: \R^m\to\R^n$, identified as a matrix $W\in\mathbb{R}^{n\times m}$, we define $\|W\|_{\op} \eqdef \sup_{x\in \R^m \setminus \{\mathbf{0}\}}\|W \mathbf{x}\|_{\infty}/\|\mathbf{x}\|_{\infty}$ to be the induced $\ell^{\infty}$-operator norm, satisfying
$
    \|W\|_{\op} = \max_{i=1,\dots,n} \sum_{j=1}^m |W_{i,j}|
$.
In our discussion of matrices generated by vertex relationships in graphs, we allow flexibility in how matrix entries are referred to. For example, when the vertices are indexed, we write $[A]_{i,j}$ or alternatively $[A]_{v_i,v_j}$. When explicit indexing is unnecessary or cumbersome, we simply write $[A]_{v,w}$ to denote an entry in $A$ associated with the vertex pair $(v,w)$.

Finally, for $A,B>0$, we write $A\lesssim B$ to indicate that there exists an absolute constant $C>0$ such that $A\leq CB$, and $A\asymp B$ to mean both $A\lesssim B$ and $B\lesssim A$.

\paragraph{Strongly connected digraph as metric space and graph Laplacian.} A digraph $G=(V,E)$ consists of a vertex set $V$ and an edge set $E \subset V \times V$, where each edge is an ordered pair indicating direction. We always assume $G$ to be \emph{simple}, i.e., it contains no self-loops and no multiple directed edges $(i,j)$. 
A digraph is strongly connected if for every $i,j\in V$, there exists a directed path from $i$ to $j$, 
and a directed path from $j$ to $i$. 
An example of a (simple) strongly connected digraph is $G_\mathrm{sc} = (V_\mathrm{sc}, E_\mathrm{sc})$ given in \eqref{awesomegraph}.

Let $G=(V,E)$ be a strongly connected digraph, with $\# V=k$. The \emph{adjacency matrix} $A_G \in \mathbb{R}^{k \times k}$ is defined by $[A_{G}]_{i,j}=1$ if $(i,j)\in E$ and $[A_{G}]_{i,j}=0$ otherwise.
The associated \emph{degree matrix} is diagonal, with $[D_G]_{i,i} \eqdef \sum_{j=1}^k [A_G]_{i,j}$.
Then for a strongly connected digraph $G$, the \emph{row-stochastic transition matrix} 
\begin{equation} \label{transition}
    P_G \eqdef D_G^{-1}A_G
\end{equation}
is \emph{irreducible}, i.e., for every $i,j\in V$, there exists $l\in\N$ such that $[P_G^l]_{i,j}>0$. 
Thus, $P$ has a unique positive left eigenvector $\phi$ with the eigenvalue $1$, normalized by \citep[Section~2]{chung2005laplacians}, 
\begin{equation} \label{Perronsum}
    \sum_{i\in V}\phi(i)=1,
\end{equation}
known as the \emph{Perron vector} or the invariant distribution of $P_G$. Let $\Phi_G\in\mathbb{R}^{k\times k}$ be diagonal such that $\Phi_G(i,i)\eqdef\phi(i)$.  
Following \citep{chung2005laplacians}, define the \emph{combinatorial} Laplacian by\footnote{One may alternatively consider the \emph{normalized} graph Laplacian \citep{chung2005laplacians}, $\Delta^\mathrm{n}_G \eqdef \Phi_G^{-1/2}\Delta_G \Phi_G^{-1/2}$. 
We do not pursue this choice here, since the $\ell^{\infty}$-operator norm 
of $\Delta^\mathrm{n}_G$ can grow rapidly, for $G=G_\mathrm{sc}$ \eqref{awesomegraph}; see \eqref{inductive1}, \eqref{inductive2} in Appendix~\ref{appx:main}.}
\begin{equation} \label{digraphLap}
    \Delta_G \eqdef \Phi_G - \frac{\Phi_G P_G + P_G^{\top}\Phi_G^{\top}}{2} = \Phi_G - \frac{\Phi_G P_G + P_G^{\top}\Phi_G}{2}
\end{equation}
where $M^{\top}$ denotes the transpose of $M$. Evidently, $\Delta_G$ is a real-valued symmetric matrix.
Let $(X_n)_{n\in\N_0}$ be a discrete-time Markov chain on $V$ with $\mathbb{P}(X_{n+1}=j| X_n=i) = [P_G]_{i,j}$.
The \emph{hitting time} for $i\in V$ is a random variable $\tau_i \eqdef \inf\{n\in\N: X_n = i\}$. 
Let $Q_G\in\mathbb{R}^{k\times k}$ be,
\begin{equation} \label{QG}
    [Q_G]_{i,j}\eqdef \mathbb{P}[\tau_j< \tau_i| X_0 = i] \quad\text{ for }\quad i,j\in [k].
\end{equation}
Let $E_G\in\mathbb{R}^{k\times k}$ be the normalized hitting probabilities matrix, such that 
\begin{equation} \label{boydspecial}
    [E_G]_{i,j}\eqdef \phi(i)[Q_G]_{i,j}=\phi(j)[Q_G]_{ji} \quad\text{ if }\quad i\not=j
\end{equation}
\citep[Lemma~1.1]{boyd2021metric} and $[E_G]_{i,i}\eqdef 1$. 
Then by \citep[Theorem~1.3]{boyd2021metric}, the \emph{hitting probability metric} 
\begin{equation} \label{hittingmetric}
    d_G(i,j) \eqdef -\log ([E_G])_{i,j}
\end{equation}
is a metric on $G$.
In general, the metric encodes the \emph{global graph connectivity} induced by the long-term behavior of the associated directed Markov chain.

\paragraph{Graph convolutional networks on digraphs.}
Let $k\in\N$, and let $\mathcal{G}_k$ be the set of simple, strongly connected digraphs on $[k]$.
For each $G=(V,E)\in\mathcal{G}_k$, where we identify $V=[k]$, let $\Delta_G$ be 
the combinatorial Laplacian in \eqref{digraphLap}. For $\mathrm{p}\in\mathbb{N}$, let $\Delta_G^{\mathrm{p}}$ be the $\mathrm{p}$-power of $\Delta_G$. We consider the following common GCN model; see~\citep[Chapter~5.3]{ma2021deep}. 

\begin{definition}
\label{defn:GGCN}
Let $L,\mathrm{p}, d_\mathrm{in}, d_\mathrm{out}\in\N$. 
Let $\beta_1,\dots,\beta_L>0$.
For $l=0,1,\dots,L$, let $d_l\in\N$, with $d_0\eqdef d_\mathrm{in}$, $d_L\eqdef d_\mathrm{out}$.
Let $E_\mathrm{in}\subset\R^{d_\mathrm{in}}$ and $E_\mathrm{out}\subset\R^{d_\mathrm{out}}$.
For $l=1,2,\dots,L$, let $W_l\in\R^{d_l\times d_{l-1}}$ be given weight matrices, with $\|W_l\|_\mathrm{op}\le \beta_l$.
Let $\sigma:\mathbb{R}\to \mathbb{R}$ be a given $1$-Lipschitz activation function. 
We consider the maps $f_\mathrm{GCN}: \mathcal{G}_k\times E_\mathrm{in}^{k} \to E_\mathrm{out}^{k}\subset \mathbb{R}^{d_\mathrm{out}\times k}$ defined by generalized GCNs with $\mathrm{p}$-hop graph convolution, activation $\sigma$, \emph{network parameters} $(W_1,\dots,W_L)$, and \emph{network size} is given by $(\beta_1,\dots,\beta_L)$. 
These maps admit the following iterative representation. For $G\in\mathcal{G}_k$ and $\mathbf{x}\in E_\mathrm{{in}}^k$, let $f_\mathrm{GCN}(G,\mathbf{x})\eqdef \mathbf{H}_L \eqdef W_L \mathbf{H}_{L-1}$, where 
\begin{align} \label{eq:GCNcompute}
    \mathbf{H}_{l+1} \eqdef \mathfrak{L}_{l+1}(\mathbf{H}_l) \quad\text{ for }\quad l=0,1,\dots,L-2, \quad\text{ and }\quad \mathbf{H}_0 \eqdef \mathbf{x}.
\end{align}
Here, $\mathfrak{L}_l(\tilde{\mathbf{x}}) \eqdef \sigma\bullet (W_l(\Delta_G^{\mathrm{p}} \tilde{\mathbf{x}}^{\top})^{\top})$, for $\tilde{\mathbf{x}}\in \mathbb{R}^{d_l\times k}$, where $\bullet$ denotes component-wise application. 
\end{definition}

\paragraph{Aitchison geometry.} For $\m\in\mathbb{N}$ with $\m\geq 2$, let $\Delta_{\m} \eqdef \{\mathbf{p}\in\mathbb{R}^{\m}: p_i\geq 0, \, \sum_{i=1}^{\m} p_i =1\}$ be the $\m$-\emph{simplex}.
The relative interior of $\Delta_{\m}$ is
\begin{equation} \label{interiorsimplex}
    \Delta_{\m}^{\circ}\eqdef \{\mathbf{p}\in\mathbb{R}^{\m}: p_i> 0 \text{ and } \sum_{i=1}^m p_i =1\}.
\end{equation}
It is well-known that 
$\Delta_{\m}^{\circ}$ is an $(\m-1)$--dimensional simplex embedded in a linear subspace admitting the \emph{Helmert basis} \citep{aitchison1982statistical}, \citep[Section~14.1]{mardia2009directional}
\begin{equation*}
    \mathbf{e}^i \eqdef \sqrt{\frac{i}{i+1}} \Big(\underbrace{\frac{1}{i},\dots,\frac{1}{i}}_{i \text{ entries}}, -1,0,\dots,0\Big) 
\quad\text{ for }\quad i=1,2,\dots,\m-1.
\end{equation*}
Inspired by \citep{greenacre2021compositional}, we further turn $\Delta_{\m}^{\circ}$ into a \emph{Hilbert space} with the \emph{Aitchison inner product}
\begin{equation*}
    \langle \mathbf{p}, \mathbf{q} \rangle_A
    \eqdef \frac{1}{2\m} \sum_{i=1}^{\m} \sum_{j=1}^{\m} \log\Big(\frac{p_i}{p_j}\Big)\log \Big(\frac{q_i}{q_j}\Big).
\end{equation*}
Then $\Delta_{\m}^{\circ}$, equipped with the \emph{Aitchison norm} $\|\cdot\|_A \eqdef \sqrt{\langle \cdot, \cdot \rangle_A}$, is a metric space with the induced \emph{Aitchison metric}
\begin{equation} \label{Ametric}
    d_A (\mathbf{p},  \mathbf{q}) = \|\mathbf{p} - \mathbf{q}\|_A.
\end{equation}
When equipped with the Aitchison geometry \citep{aitchison1982statistical}, $\Delta_{\m}^{\circ}$ is isometrically isomorphic to $\mathbb{R}^{\m-1}$ with its Euclidean metric \citep{pawlowsky2006compositional}. 
We briefly record this identification.
Let $\mathbb{H}^{\m-1}\eqdef \{\mathbf{x}\in\mathbb{R}^{\m}: \sum_{i=1}^{\m} x_i = 0\}$.
Let $\mathrm{clr}: \Delta_{\m}^{\circ} \to\mathbb{H}^{\m-1}$ denote the \emph{centered log-ratio} transform that is
\begin{equation} \label{clr}
    \mathrm{clr}(\mathbf{p}) \eqdef \Big(\log p_i - \frac{1}{\m} \sum_{j=1}^{\m} \log p_j \Big)_{i=1}^\mathtt{m}.
\end{equation}
Projecting $\mathrm{clr}$ onto the Helmert basis defines the \emph{isometric log-ratio} map $\mathrm{ilr}: (\Delta_{\m}^{\circ}, d_A) \to (\mathbb{R}^{\m-1}, d_2)$, 
\begin{equation} \label{ilr}
    \mathrm{ilr}(\mathbf{p}) \eqdef \big(\langle \mathrm{clr}(\mathbf{p}), \mathbf{e}^i\rangle_2\big)_{i=1}^{\m-1},
\end{equation}
where $\langle \cdot,\cdot\rangle_2$ denotes the standard Euclidean inner product and $d_2$ the usual Euclidean distance.
We henceforth freely identify $\Delta_{\m}^{\circ}$ with the Euclidean space $\mathbb{R}^{\m-1}$ for convenience.

\section{Setup and main result} \label{s:Main}

Recall the computation loop $G_\mathrm{sc} = (V_\mathrm{sc}, E_\mathrm{sc})$ introduced in \eqref{awesomegraph} and the associated perfect rooted $\nu$-ary tree $\mathtt{B}_{\nu}=(V_{\nu}, E_{\nu})$ of height $\mathtt{h}$. 
The total number of vertices in $V_\mathrm{sc}$ is
\begin{equation} \label{totalvertex}
    \#V_\mathrm{sc} = k = \frac{2\nu^{\h+1}-\nu^{\h}-1}{\nu-1}.
\end{equation}
We fix an identification of $V_\mathrm{sc}$ with the index set $[k]$, chosen so that the first $1+\nu+\nu^2+\dots+\nu^{\h}$, or $k - \nu^{\h}$,
indices correspond to the nodes $V_{\nu}$ of $\mathtt{B}_{\nu}$, ordered hierarchically from the base level upward.
One convenient choice is
\begin{equation} \label{convenient}
    \underbrace{1, 2,\dots,\nu^{\h}}_{\text{base nodes}}, \underbrace{\nu^{\h}+1, \nu^{\h}+2,\dots,\nu^{\h}+\nu^{\h-1}}_{\text{level $1$ nodes}},\dots,\underbrace{k - \nu^{\h}}_{\text{root } \mathtt{r}},
\end{equation}
where the base node $v_i$ is connected to the tape position $T_{i-1}$, for $i=1,2,\dots,\nu^{\h}$.
We allow mixed notation, in which we continue to refer to the root as $\mathtt{r}$ and the tape-nodes as $T_i$.
Let $\Gamma \eqdef V_{\nu}\setminus\{v_1,v_2,\dots,v_{\nu^{\h}}\}$, consisting precisely of the computation nodes in $\mathtt{B}_{\nu}$.
Let $\mathfrak{G} = \{\mathfrak{g}^1,\dots,\mathfrak{g}^{\m}\}$ be a finite gate set.
A \emph{gate configuration} is a mapping $\mathscr{C}: v\mapsto \mathfrak{g}_v$ from $\Gamma$ into $\mathfrak{G}$.
We assume this configuration is given but unknown.
Let $\eta \in (0,1)$ be a certainty level. To avoid triviality, we exclude the case of perfect certainty $\eta = 1$ and no certainty $\eta=0$.
The \emph{reasoning probe operation} $Q_{\eta}:V\rightarrow \Delta_{\m}^{\circ}$ returns a noisy categorical description at a node $v\in V$ defined by
\begin{equation*}
    Q_{\eta}(v)
    \eqdef
    \Big( \eta\cdot\mathbf{1}_{\{\mathfrak{g}_v = \mathfrak{g}^i\}} + \frac{1-\eta}{\m-1} \cdot \mathbf{1}_{\{\mathfrak{g}_v \not= \mathfrak{g}^i\}} \Big)_{i=1}^{\m}.
\end{equation*}
Thus $Q_{\eta}$ provides local but uncertain information about $\mathscr{C}$. 
We assume a \emph{bounded-complexity condition} on $Q_{\eta}$, namely
\begin{equation} \label{eq:boundedcomp}
    \|Q_{\eta}(v) - Q_{\eta}(w)\|_A \leq K_{\h}, \quad\text{ for }\quad v,w\in\Gamma,
\end{equation}
where $K_{\h}$ grows at most exponentially in terms of $\h$.
Let $\mu_{[k]}$ be a probability measure on $[k]$ with support equal to $\Gamma$. 
Let\footnote{This means that $\mu$ is the push-forward of $\mu_{[k]}$ under $(\mathbbm{1}\times Q_{\eta})$.} 
$\mu \eqdef (\mathbbm{1}\times Q_{\eta})_{\#}\mu_{[k]}$, and suppose we are given independent random samples\footnote{We assume sampling with replacement and that the sampling distribution $\mu$ is well-behaved, in the sense that no single point carries disproportionately arbitrarily large mass. See Section~\ref{s:Main__ss:Discussion}.}
\begin{equation} \label{iidrv}
    (V_1,\mathbf{P}_1),\dots,(V_N,\mathbf{P}_N)\sim\mu,
\end{equation}
taking values on $\Gamma\times \Delta_{\m}^{\circ}$; that is, $\mathbf{P}_i = Q_{\eta}(V_i)$. We write $\mu^N$ to denote the empirical version of $\mu$.
Let $J: \Delta_{\m}^{\circ}\times \Delta_{\m}^{\circ} \rightarrow \mathbb{R}_{\geq 0}$ be a loss function satisfying
\begin{equation} \label{ellLip}
    |J(\mathbf{y}, \mathbf{z}) - J(\mathbf{y}', \mathbf{z}')| \leq \mathtt{C}_J \max\{\|\mathbf{y}-\mathbf{y}'\|_A, \|\mathbf{z}-\mathbf{z}'\|_A\},
\end{equation}
for some $\mathtt{C}_J>0$.
For $\alpha\in (0,1)$, let $J_{\alpha}: \Delta_{\m}^{\circ}\times \Delta_{\m}^{\circ} \rightarrow \mathbb{R}_{\geq 0}$ by $J_{\alpha}(\mathbf{y}, \mathbf{z}) \eqdef J(\mathbf{y}, \mathbf{z})^{\alpha}$ be a \emph{snowflaked} loss. 
Note that the resulting loss increases sensitivity to discrepancies at fine scales. 
When $\alpha=1/2$, the resulting square-root loss has been used in \citep{stucky2017sharp, belloni2011square}.
Given a configuration $\mathscr{C}$, we assume that at time $t\in\mathbb{N}$, a temporal input is applied at the base nodes of $\mathtt{B}_{\nu}$. This input propagates through $\mathtt{B}_{\nu}$, generating a Boolean output at each $v\in\Gamma$. We collect these outputs into a vector $\mathbf{x}\in\{0,1\}^{s_{\nu,\h}}$ where $s_{\nu,\h} \eqdef \frac{\nu^{\h}-1}{\nu-1}$.
\emph{Treating each such time step as a stand-alone experimental trial induced by the underlying configuration $\mathscr{C}$}, we study the (static) \emph{transductive generalization gap} over a set of graph learners trained on $\mathbf{x}$ at time $t\in\mathbb{N}$ and values of $Q_{\eta}$ at sampled computation nodes. 

Let $G_{\Gamma}\subset G_\mathrm{sc}$ be the induced subgraph on $\Gamma$.
Recall the general GCN model $f_\mathrm{GCN}$ in Definition~\ref{defn:GGCN}, with depth $L\in\mathbb{N}$ and network parameters $(\beta_1,\dots,\beta_L)$.
Set $d_\mathrm{in}=1$, $d_\mathrm{out}=\mathtt{m}-1$. 
We define a hypothesis set $\mathcal{H}$ for $Q_{\eta}$ consisting of
\begin{equation} \label{h}
    h \eqdef (\mathbf{1}_{s_{\nu,\h}}\otimes \mathrm{ilr}^{-1}) \circ f_\mathrm{GCN}(G_\Gamma,\cdot),
\end{equation}
where $\mathbf{1}_{s_{\nu,\h}}$ denotes the $s_{\nu,\h}$-dimensional vector of ones, and $\mathrm{ilr}$ is given in \eqref{ilr}.
By definition \eqref{h}, $E_\mathrm{in}=\{0,1\}$ and $E_\mathrm{out}\subset \Delta_{\m}^{\circ}$.
For $h\in\mathcal{H}$, we take the empirical risk to be
\begin{equation} \label{eqdef:emprisk}
    \mathcal{R}_{\mathbf{x},t}^{\alpha, N}(h)  \eqdef \frac1{N} \sum_{i=1}^N
    J_{\alpha}(\pi_{V_i}\circ h,\mathbf{P}_i),
\end{equation}
and the corresponding population risk to be
\begin{equation} \label{eqdef:truerisk}
    \mathcal{R}_{\mathbf{x},t}^{\alpha}
    (h) \eqdef \mathbb{E}_{(V,\mathbf{P})\sim\mu} \big[J_{\alpha}(\pi_{V}\circ h,\mathbf{P})\big].
\end{equation}
The \emph{worst-case} discrepancy between these two risks is captured by the transductive generalization gap over $\mathcal{H}$ defined by
\begin{equation}
\label{eq:gen_gap}
    \sup_{h\in\mathcal{H}} 
    \big|\mathcal{R}_{\mathbf{x},t}^{\alpha}(h) - \mathcal{R}_{\mathbf{x},t}^{\alpha, N}(h) \big|.
\end{equation}
Our main result provides an explicit upper bound on the generalization gap in terms of the problem parameters.
\begin{theorem}[Main result]
\label{thrm:main_result}
Let $\alpha\in (0,1)$.
Let $t,N\in \mathbb{N}$. For every $\delta\in (0,1)$, the following event holds with probability at least $1-\delta$
\begin{multline} \label{mainres}
    \sup_{h\in\mathcal{H}} |\mathcal{R}_{\mathbf{x},t}^{\alpha}(h) - \mathcal{R}_{\mathbf{x},t}^{\alpha, N}(h)|
    \\
    \lesssim \Big(\mathtt{C}_J \Big(\max\Big\{\mathtt{m}^{1/2}\Big(\frac{3+\nu}{2}\Big)^{\mathrm{p} (L-1)} \prod_{l=1}^L \beta_l, \nu^{\h}
    , K_{\h} \Big\}\Big)^{5/2} \Big)^{\alpha} \cdot
    \Big(\frac{1}{\sqrt{N}} + \frac{\sqrt{\log(2/\delta)}}{\sqrt{N}}\Big).
\end{multline}
\end{theorem}

\section{Interpretation and proof strategy of main theorem}
\label{s:Main__ss:Discussion}

We now provide a discussion of the main result, including its interpretation and proof strategy.

\paragraph{Geometric control via graph diameter.}
The generalization bound in Theorem~\ref{thrm:main_result} is governed by the diameter of the induced metric on the computation graph.  For looped $\nu$-ary trees, this diameter is, up to constants, determined by the tree height and enters the concentration bounds directly.  This suggests that, for more general looped strongly connected computation graphs, the underlying digraph diameter may similarly control statistical efficiency.

\paragraph{Uniform graph-size-independent generalization and the role of $\alpha$.}
The analysis underlying Theorem~\ref{thrm:main_result} is \emph{time-uniform} and applies to all $\mathbf{x}\in\{0,1\}^{s_{\nu,\h}}$ without leveraging accumulated information across rounds\footnote{Learning under feedback—where inputs depend on past outputs—would require additional assumptions on the permissible configurations.}. A key feature of the result is that the parameter $\alpha\in(0,1)$ in \eqref{mainres} is chosen \emph{a priori} and controls a tradeoff between the strength of the loss and the dependence of the generalization bound on the graph structure. In particular, the loss $J_\alpha$ varies with $\alpha$, so different choices of $\alpha$ correspond to evaluating performance under different metrics. For fixed $\alpha$, the constants in the bound may depend on structural parameters such as $\h$. However, selecting $\alpha\asymp \h^{-1}$ ensures that these constants remain uniformly controlled as $\h$ grows, yielding meaningful convergence with a rate that remains $\mathcal{O}(1/\sqrt{N})$.
\emph{To our knowledge, this is the first instance in which such dimension- or cardinality-free rates are established for reasoning probes operating on looped, strongly connected computation graphs.}

In our setting, the introduction of the snowflaked loss $J_\alpha$ is not merely technical: it enables a low-distortion embedding of the computation graph into one dimension, which is essential for obtaining optimal rates via Wasserstein concentration. Smaller values of $\alpha$ improve the geometric regularity of the graph metric and facilitate this embedding, while larger values correspond to stronger discrepancies between probe outputs. 
\emph{This tradeoff is intrinsic to analyses based on snowflaked metrics; here, it establishes a link between the geometry of the underlying computation graph and the statistical guarantees that can be achieved.}

\paragraph{Contextual discussion on coverage property of i.i.d. sampling}
With computational efficiency in mind, we consider i.i.d.\ (memoryless) probing of vertices. Unlike systematic exploration or uniform sampling without replacement---which enforces non-revisitation and guarantees full coverage after $k$ queries at the cost of storing previously visited vertices---i.i.d.\ sampling incurs no storage overhead. Under this scheme, the total number of probes $N$ may exceed $\nu^{\h}$, the order of the total number of computational nodes, and even satisfy $N\gg \nu^{\h}$ without guaranteeing full coverage. In particular, repeated sampling of a small subset of vertices may dominate, leading to highly unbalanced visitation frequencies. This phenomenon is classical in occupancy problems; see, for example, \citep[Proposition~1]{BrownPekozRoss2008Coupon}.
While arbitrarily ill-conditioned distributions are theoretically admissible, in practice users rarely deploy strategies assigning vanishing mass to a large fraction of vertices. 
(\emph{We emphasize that our generalization bounds do not require uniform coverage or coupon-collector–type guarantees, and remain valid even when many vertices are never observed.})

\paragraph{Learning-theoretic implications and future directions.}
Our results are a step toward transductive learning guarantees for reasoning probes.  Future work includes extending the analysis to broader looped reasoning models and probe architectures beyond GCNs, and developing methods for recovering structured internal computations from partial observations.  Our guarantees identify when global structure can be inferred from limited queries under minimal assumptions; extending this to finer inference and detection tasks remains open.

\subsection{A metric-embedding-based proof strategy} \label{s:proofstrategy}

Our proof adopts a geometric strategy inspired by recent developments in metric embedding methods in learning theory. 
We briefly review related ideas from the literature, and then outline the proof strategy.
Viewing both $\mu$ and $\mu^N$ as measures on a finite metric space, the central idea is to relate the worst-case generalization gap---uniformly over the hypothesis class---between the true risk and the empirical risk to a Wasserstein distance between $\mu$ and $\mu^N$. 
This perspective frames generalization as a statistical optimal transport problem, where the main technical challenge is to control the concentration of the Wasserstein distance.
Existing results are predominantly qualitative, including a central limit theorem \citep{sommerfeld2018inference}.
Few non-asymptotic guarantees are available; see e.g. \citep{kratsios2024tighter}.
There, the Wasserstein distance between $\mu$ and $\mu^N$ is mapped to a corresponding Wasserstein distance in Euclidean space, for which classical optimal transport theory provides sharp concentration results. Specifically, for $m\geq 3$, the Wasserstein distance tightly concentrates around its expectation, which scales as $\mathcal{O}(1/N^{1/m})$; for $m=2$, the rate is $\mathcal{O}(\log_2(N)/N^{1/2})$; and for $m=1$, the optimal rate $\mathcal{O}(1/N^{1/2})$ holds.
The resulting generalization bounds in \citep{kratsios2024tighter}, akin to standard Occam-type bounds, depend on the cardinality of the underlying metric space. This hierarchy of rates highlights the intrinsic cost of high dimensionality and motivates reducing the effective dimension.

More recently, \citep{detering2025learning} removes the dependence on metric cardinality by adopting a fractional $1/2$-Wasserstein distance.
In their approach, both the true and empirical risks are measured with respect to the $1/2$-snowflaked loss function, and---by combining the results of \citep{hou2023instance} with Kantorovich duality---the resulting generalization gap is controlled by a multiple of the corresponding $1/2$-Wasserstein distance in Euclidean space.
The analysis therein incurs a logarithmic factor in $N$, arising from a reduction to two dimensions; comparable rates are well-known in classical Rademacher- and Gaussian-complexity bounds~\citep{bartlett2002rademacher,bartlett2005local,bartl2025uniform,bartl2025we}.  
We follow the same general strategy, but observe that the relevant metric space admits an $\alpha$-snowflake representation in one dimension, for $\alpha\in (0,1)$. We show that the graph metric space $G_{\Gamma}$, viewed as a subgraph of $G_{\rm sc}$ and equipped with the \emph{snowflaked} hitting-probability metric (see \eqref{hittingmetric}) $d_{G_{\rm sc}}^{\alpha}$, embeds into $(\mathbb{R},d_{\infty})$, where $d_{\infty}$ is the absolute value metric. This permits access to the optimal one-dimensional Wasserstein convergence rate of $\mathcal{O}(1/N^{1/2})$.

\paragraph{Insights from the proof of Theorem~\ref{thrm:main_result}}  The proof proceeds in two steps: 
\begin{itemize}
    \item[{\bf Step 1:}] We construct a one-dimensional bi-Lipschitz embedding of $(G_{\Gamma}, d_{G_{\rm sc}}^{\alpha})$ with minimal distortion (given in Proposition~\ref{prop:independent embedding} for general finite metric space).  
    This step combines Assouad-type embedding techniques for doubling metric spaces \citep{OferLowDimEmbeddingDoublingSpaces} with low-dimensional embedding results for snowflaked metrics \citep{HarPeledMender_2006EmbeddingViaFastNets_SIAMJCompute}, which can be viewed as an ultra-low-dimensional analogue of Assouad-type constructions for doubling metric spaces~\citep{assouad1983plongements,naor2012assouad}. 
    
    \item[{\bf Step 2:}] Following the reduction in Step 1, we control the worst-case generalization gap \eqref{eq:gen_gap} by the corresponding one-dimensional $\alpha$-Wasserstein distance (Proposition~\ref{prop:con_holder_wass}), yielding a non-asymptotic concentration analysis complementary to~\citep{sommerfeld2018inference}. 
    As an intermediate step, we estimate the Lipschitz regularity of the model $f_{\rm GCN}$ (Proposition~\ref{prop:est}) and diameter of the digraph $G_{\Gamma}$; both quantities contribute, in part, to the majorant constant in \eqref{mainres}.
    Combined, these yield optimal sample complexity bounds with controlled constants, refining the geometric learning framework of~\citep{detering2025learning,kratsios2024tighter}.
\end{itemize}




\newpage
\bibliographystyle{plain}
\bibliography{0_DAG}

@inproceedings{ma2022is,
title={Is Homophily a Necessity for Graph Neural Networks?},
author={Yao Ma and Xiaorui Liu and Neil Shah and Jiliang Tang},
booktitle={International Conference on Learning Representations},
year={2022},
url={https://openreview.net/forum?id=ucASPPD9GKN}
}

@article{pawlowsky2006compositional,
  title={Compositional data and their analysis: an introduction},
  author={Pawlowsky-Glahn, Vera and Egozcue, Juan Jos{\'e}},
  journal={Geological Society, London, Special Publications},
  volume={264},
  number={1},
  pages={1--10},
  year={2006},
  publisher={The Geological Society of London}
}

@article{pal2020multiplicative,
  title={Multiplicative Schr{\"o}dinger problem and the Dirichlet transport},
  author={Pal, Soumik and Wong, Ting-Kam Leonard},
  journal={Probability Theory and Related Fields},
  volume={178},
  number={1},
  pages={613--654},
  year={2020},
  publisher={Springer}
}

@incollection{erb2021information,
  title={The information-geometric perspective of Compositional Data Analysis},
  author={Erb, Ionas and Ay, Nihat},
  booktitle={Advances in Compositional Data Analysis: Festschrift in Honour of Vera Pawlowsky-Glahn},
  pages={21--43},
  year={2021},
  publisher={Springer}
}

@book{mardia2009directional,
  title={Directional statistics},
  author={Mardia, Kanti V and Jupp, Peter E},
  year={2009},
  publisher={John Wiley \& Sons}
}

@article{stucky2017sharp,
  title={Sharp oracle inequalities for square root regularization},
  author={Stucky, Benjamin and Van De Geer, Sara},
  journal={Journal of Machine Learning Research},
  volume={18},
  number={67},
  pages={1--29},
  year={2017}
}

@article{belloni2011square,
  title={Square-root lasso: pivotal recovery of sparse signals via conic programming},
  author={Belloni, Alexandre and Chernozhukov, Victor and Wang, Lie},
  journal={Biometrika},
  volume={98},
  number={4},
  pages={791--806},
  year={2011},
  publisher={Oxford University Press}
}

@article{kratsios2024tighter,
  title={Tighter generalization bounds on digital computers via discrete optimal transport},
  author={Kratsios, Anastasis and Neuman, A Martina and Pammer, Gudmund},
  journal={arXiv preprint arXiv:2402.05576},
  year={2024}
}

@article{sommerfeld2018inference,
  title={Inference for empirical Wasserstein distances on finite spaces},
  author={Sommerfeld, Max and Munk, Axel},
  journal={Journal of the Royal Statistical Society Series B: Statistical Methodology},
  volume={80},
  number={1},
  pages={219--238},
  year={2018},
  publisher={Oxford University Press}
}

@inproceedings{craven1996trepan,
 author = {Craven, Mark and Shavlik, Jude},
 booktitle = {Advances in Neural Information Processing Systems},
 editor = {D. Touretzky and M.C. Mozer and M. Hasselmo},
 pages = {},
 publisher = {MIT Press},
 title = {Extracting Tree-Structured Representations of Trained Networks},
 url = {https://proceedings.neurips.cc/paper_files/paper/1995/file/45f31d16b1058d586fc3be7207b58053-Paper.pdf},
 volume = {8},
 year = {1995}
}

@article{contreras2022dexire,
AUTHOR = {Contreras, Victor and Marini, Niccolo and Fanda, Lora and Manzo, Gaetano and Mualla, Yazan and Calbimonte, Jean-Paul and Schumacher, Michael and Calvaresi, Davide},
TITLE = {A DEXiRE for Extracting Propositional Rules from Neural Networks via Binarization},
JOURNAL = {Electronics},
VOLUME = {11},
YEAR = {2022},
NUMBER = {24},
ARTICLE-NUMBER = {4171},
URL = {https://www.mdpi.com/2079-9292/11/24/4171},
ISSN = {2079-9292},
DOI = {10.3390/electronics11244171}
}

@article{Olah2018building,
  title   = {The Building Blocks of Interpretability},
  author  = {Olah, Chris and Satyanarayan, Arvind and Johnson, Ian 
             and Carter, Shan and Schubert, Ludwig and Ye, Katherine 
             and Mordvintsev, Alexander},
  journal = {Distill},
  volume  = {3},
  number  = {3},
  pages   = {e10},
  year    = {2018},
  doi     = {10.23915/DISTILL.00010}
}

@article{Olah2020zoom,
  title   = {Zoom In: An Introduction to Circuits},
  author  = {Olah, Chris and Cammarata, Nick and Schubert, Ludwig and 
             Goh, Gabriel and Petrov, Michael and Carter, Shan},
  journal = {Distill},
  volume  = {5},
  number  = {3},
  pages   = {e00024.001},
  year    = {2020},
  doi     = {10.23915/distill.00024.001}
}

@inproceedings{dai-etal-2022-knowledge,
  title = {Knowledge Neurons in Pretrained Transformers},
  author = {Dai, Damai and Dong, Li and Hao, Yaru and Sui, Zhifang and Chang, Baobao and Wei, Furu},
  booktitle = {Proceedings of the 60th Annual Meeting of the Association for Computational Linguistics (Volume 1: Long Papers)},
  year = {2022},
  pages = {8493--8502},
  address = {Dublin, Ireland},
  publisher = {Association for Computational Linguistics}
}

@article{Alain2016linearprobes,
  title = {Understanding intermediate layers using linear classifier probes},
  author = {Alain, Guillaume and Bengio, Yoshua},
  journal = {arXiv preprint arXiv:1610.01644},
  year = {2016}
}

@article{DBLPjournals/corr/abs-1805-01070,
  author    = {Alexis Conneau and Germ{\'{a}}n Kruszewski and Guillaume Lample 
               and Lo{\"{\i}}c Barrault and Marco Baroni},
  title     = {What you can cram into a single vector: Probing sentence embeddings 
               for linguistic properties},
  journal   = {CoRR},
  volume    = {abs/1805.01070},
  year      = {2018},
  eprint    = {1805.01070}
}

@inproceedings{bau2017network,
  title = {Network Dissection: Quantifying Interpretability of Deep Visual Representations},
  author = {Bau, David and Zhou, Bolei and Khosla, Aditya and Oliva, Aude and Torralba, Antonio},
  booktitle = {Proceedings of the IEEE Conference on Computer Vision and Pattern Recognition (CVPR)},
  pages = {3319--3327},
  year = {2017},
  doi = {10.1109/CVPR.2017.354}
}

@inproceedings{saunshi2025reasoning,
title={Reasoning with Latent Thoughts: On the Power of Looped Transformers},
author={Nikunj Saunshi and Nishanth Dikkala and Zhiyuan Li and Sanjiv Kumar and Sashank J. Reddi},
booktitle={The Thirteenth International Conference on Learning Representations},
year={2025},
url={https://openreview.net/forum?id=din0lGfZFd}
}

@article{LiKratsiosGhoukasianZvigelsky_CertifiableReasoningUniversal,
  title   = {Certifiable Reasoning Is Universal: A Differentiable Reasoning {AI}},
  author  = {Li, Wenhao and Kratsios, Anastasis and Ghoukasian, Hrad and Zvigelsky, Dennis},
  journal = {arXiv preprint},
  year    = {2026},
  note    = {Preprint}
}

@article{bengio2024managing,
  title={Managing extreme {AI} risks amid rapid progress},
  author={Bengio, Yoshua and Hinton, Geoffrey and Yao, Andrew and Song, Dawn and Abbeel, Pieter and Darrell, Trevor and Harari, Yuval Noah and Zhang, Ya-Qin and Xue, Lan and Shalev-Shwartz, Shai and others},
  journal={Science},
  volume={384},
  number={6698},
  pages={842--845},
  year={2024},
  publisher={American Association for the Advancement of Science}
}

@article{hou2023instance,
  title={Instance-dependent generalization bounds via optimal transport},
  author={Hou, Songyan and Kassraie, Parnian and Kratsios, Anastasis and Krause, Andreas and Rothfuss, Jonas},
  journal={Journal of Machine Learning Research},
  volume={24},
  number={349},
  pages={1--51},
  year={2023}
}

@article{naor2012assouad,
  title={Assouad’s theorem with dimension independent of the snowflaking},
  author={Naor, Assaf and Neiman, Ofer},
  journal={Revista Matematica Iberoamericana},
  volume={28},
  number={4},
  pages={1123--1142},
  year={2012}
}

@article{assouad1983plongements,
  title={Plongements {L}ipschitziens dans $\mathbb{R}^n$},
  author={Assouad, Patrice},
  journal={Bulletin de la Soci{\'e}t{\'e} Math{\'e}matique de France},
  volume={111},
  pages={429--448},
  year={1983}
}

@book{bostrom2014superintelligence,
  title={Superintelligence: Paths, Dangers, Strategies},
  author={Bostrom, Nick},
  publisher={Oxford University Press},
  year={2014}
}

@inproceedings{asai2024selfrag,
title={Self-{RAG}: Learning to Retrieve, Generate, and Critique through Self-Reflection},
author={Akari Asai and Zeqiu Wu and Yizhong Wang and Avirup Sil and Hannaneh Hajishirzi},
booktitle={The Twelfth International Conference on Learning Representations},
year={2024},
url={https://openreview.net/forum?id=hSyW5go0v8}
}

@article{lewis2020retrieval,
  title={Retrieval-augmented generation for knowledge-intensive nlp tasks},
  author={Lewis, Patrick and Perez, Ethan and Piktus, Aleksandra and Petroni, Fabio and Karpukhin, Vladimir and Goyal, Naman and K{\"u}ttler, Heinrich and Lewis, Mike and Yih, Wen-tau and Rockt{\"a}schel, Tim and others},
  journal={Advances in neural information processing systems},
  volume={33},
  pages={9459--9474},
  year={2020}
}

@article{song2025thinking,
  title={Thinking Isn't an Illusion: Overcoming the Limitations of Reasoning Models via Tool Augmentations},
  author={Song, Zhao and Yue, Song and Zhang, Jiahao},
  journal={arXiv preprint arXiv:2507.17699},
  year={2025}
}

@inproceedings{bournez2025universal,
  title     = {A Universal Uniform Approximation Theorem for Neural Networks},
  author    = {Olivier Bournez and Johanne Cohen and Adrian Wurm},
  booktitle = {Proceedings of the International Symposium on Mathematical Foundations of Computer Science (MFCS)},
  year      = {2025},
  address   = {Palaiseau, France},
  institution = {Institut Polytechnique de Paris, Université Paris-Saclay, BTU Cottbus-Senftenberg}
}

@inproceedings{maass1991sigmoidvsbool,
  author    = {Wolfgang Maass and Georg Schnitger and Eduardo D. Sontag},
  title     = {On the Computational Power of Sigmoid versus Boolean Threshold Circuits},
  booktitle = {Proceedings of the 32nd Annual Symposium on Foundations of Computer Science (FOCS)},
  year      = {1991},
  pages     = {767--776}
}

@article{merrill2023parallelism,
  author  = {William Merrill and Ashish Sabharwal},
  title   = {The Parallelism Tradeoff: Limitations of Log-Precision Transformers},
  journal = {Transactions of the Association for Computational Linguistics},
  year    = {2023}
}

@book{siegelmann2012neural,
  title={Neural networks and analog computation: beyond the Turing limit},
  author={Siegelmann, Hava T},
  year={2012},
  publisher={Springer Science \& Business Media}
}

@inproceedings{perez2018on,
    title={On the Turing Completeness of Modern Neural Network Architectures},
    author={Jorge Pérez and Javier Marinković and Pablo Barceló},
    booktitle={International Conference on Learning Representations},
    year={2019},
    url={https://openreview.net/forum?id=HyGBdo0qFm},
}

@article{chung2021turing,
  title={Turing completeness of bounded-precision recurrent neural networks},
  author={Chung, Stephen and Siegelmann, Hava},
  journal={Advances in neural information processing systems},
  volume={34},
  pages={28431--28441},
  year={2021}
}

@inproceedings{chen2025magicore,
  title     = {MAGICORE: Multi-Agent, Iterative, Coarse-to-Fine Refinement for Reasoning},
  author    = {Chen, Justin Chih-Yao and Prasad, Archiki and Saha, Swarnadeep and Stengel-Eskin, Elias and Bansal, Mohit},
  booktitle = {Proceedings of EMNLP},
  year      = {2025},
  url       = {https://aclanthology.org/2025.emnlp-main.1660/}
}

@inproceedings{wang2023selfconsistency,
  title     = {Self-Consistency Improves Chain of Thought Reasoning in Language Models},
  author    = {Wang, Xuezhi and Wei, Jason and Schuurmans, Dale and Le, Quoc V. and Chi, Ed H. and Narang, Sharan and Chowdhery, Aakanksha and Zhou, Denny},
  booktitle = {International Conference on Learning Representations (ICLR)},
  year      = {2023},
  url       = {https://openreview.net/forum?id=1PL1NIMMrw}
}

@inproceedings{shinn2023reflexion,
  title     = {Reflexion: Language Agents with Verbal Reinforcement Learning},
  author    = {Shinn, Noah and Labash, Beck and Gopinath, Ashwin and Narasimhan, Karthik R.},
  booktitle = {Advances in Neural Information Processing Systems (NeurIPS)},
  year      = {2023},
  url       = {https://arxiv.org/abs/2303.11366}
}

@article{madaan2023self,
  title={Self-refine: Iterative refinement with self-feedback},
  author={Madaan, Aman and Tandon, Niket and Gupta, Prakhar and Hallinan, Skyler and Gao, Luyu and Wiegreffe, Sarah and Alon, Uri and Dziri, Nouha and Prabhumoye, Shrimai and Yang, Yiming and others},
  journal={Advances in Neural Information Processing Systems},
  volume={36},
  pages={46534--46594},
  year={2023}
}

@inproceedings{wang2024voyager,
  title     = {Voyager: An Open-Ended Embodied Agent with Large Language Models},
  author    = {Wang, Guanzhi and Wang, Yu and Liu, Xinyi and Liu, Shibin and Chen, Yizhou and Zhang, Yiming and Zhao, Jian and Zhang, Tong},
  booktitle = {International Conference on Learning Representations (ICLR)},
  year      = {2024},
  url       = {https://arxiv.org/abs/2305.16291}
}

@inproceedings{zelikman2022star,
  title={Star: Self-taught reasoner},
  author={Zelikman, Eric and Wu, Yuhuai and Goodman, Noah D},
  booktitle={Proceedings of the NIPS},
  volume={22},
  year={2022}
}

@article {HarPeledMender_2006EmbeddingViaFastNets_SIAMJCompute,
    AUTHOR = {Har-Peled, Sariel and Mendel, Manor},
     TITLE = {Fast construction of nets in low-dimensional metrics and their
              applications},
   JOURNAL = {SIAM Journal on Computing},
  FJOURNAL = {SIAM Journal on Computing},
    VOLUME = {35},
      YEAR = {2006},
    NUMBER = {5},
     PAGES = {1148--1184},
      ISSN = {0097-5397,1095-7111},
   MRCLASS = {68W25 (52C17 68P05 68W40)},
  MRNUMBER = {2217141},
MRREVIEWER = {Tamas\ Lengyel},
       DOI = {10.1137/S0097539704446281},
       URL = {https://doi.org/10.1137/S0097539704446281},
}

@article {OferLowDimEmbeddingDoublingSpaces,
    AUTHOR = {Neiman, Ofer},
     TITLE = {Low dimensional embeddings of doubling metrics},
   JOURNAL = {Theory Comput. Syst.},
  FJOURNAL = {Theory of Computing Systems},
    VOLUME = {58},
      YEAR = {2016},
    NUMBER = {1},
     PAGES = {133--152},
      ISSN = {1432-4350,1433-0490},
   MRCLASS = {46B85 (05C10 68U05 68W25)},
  MRNUMBER = {3438642},
       DOI = {10.1007/s00224-014-9567-3},
       URL = {https://doi.org/10.1007/s00224-014-9567-3},
}

@book{villani2009optimal,
  title={Optimal transport: old and new},
  author={Villani, C{\'e}dric and others},
  volume={338},
  year={2009},
  publisher={Springer}
}

@article{BrownPekozRoss2008Coupon,
  title   = {Coupon Collecting},
  author  = {Brown, Mark and Pek{\"o}z, Erol A. and Ross, Sheldon M.},
  journal = {Probability in the Engineering and Informational Sciences},
  volume  = {22},
  number  = {2},
  pages   = {221--229},
  year    = {2008},
  publisher = {Cambridge University Press}
}

@book{jukna2012boolean,
  title={Boolean function complexity: advances and frontiers},
  author={Jukna, Stasys and others},
  volume={27},
  year={2012},
  publisher={Springer}
}

@article{chiang2025transformers,
title={Transformers in Uniform {TC}\${\textasciicircum}0\$},
author={David Chiang},
journal={Transactions on Machine Learning Research},
issn={2835-8856},
year={2025},
url={https://openreview.net/forum?id=ZA7D4nQuQF},
note={}
}

@article{bartlett2005local,
  title={LOCAL RADEMACHER COMPLEXITIES},
  author={Bartlett, Peter L and Bousquet, Olivier and Mendelson, Shahar},
  journal={The Annals of Statistics},
  volume={33},
  number={4},
  pages={1497--1537},
  year={2005}
}

@article{bartlett2002rademacher,
  title={Rademacher and gaussian complexities: Risk bounds and structural results},
  author={Bartlett, Peter L and Mendelson, Shahar},
  journal={Journal of machine learning research},
  volume={3},
  number={Nov},
  pages={463--482},
  year={2002}
}

@article{bartl2025uniform,
  title={Uniform mean estimation via generic chaining},
  author={Bartl, Daniel and Mendelson, Shahar},
  journal={arXiv preprint arXiv:2502.15116},
  year={2025}
}

@article{bartl2025we,
  title={Do we really need the {R}ademacher complexities?},
  author={Bartl, Daniel and Mendelson, Shahar},
  journal={arXiv preprint arXiv:2502.15118},
  year={2025}
}

@article{detering2025learning,
  title={Learning from one graph: transductive learning guarantees via the geometry of small random worlds},
  author={Detering, Nils and Galimberti, Luca and Kratsios, Anastasis and Livieri, Giulia and Neuman, A Martina},
  journal={arXiv preprint arXiv:2509.06894},
  year={2025}
}

@article{aitchison1982statistical,
  title={The statistical analysis of compositional data},
  author={Aitchison, John},
  journal={Journal of the Royal Statistical Society: Series B (Methodological)},
  volume={44},
  number={2},
  pages={139--160},
  year={1982},
  publisher={Wiley Online Library}
}

@article{boyd2021metric,
  title={A metric on directed graphs and Markov chains based on hitting probabilities},
  author={Boyd, Zachary M and Fraiman, Nicolas and Marzuola, Jeremy and Mucha, Peter J and Osting, Braxton and Weare, Jonathan},
  journal={SIAM Journal on Mathematics of Data Science},
  volume={3},
  number={2},
  pages={467--493},
  year={2021},
  publisher={SIAM}
}

@article{chung2005laplacians,
  title={Laplacians and the Cheeger inequality for directed graphs},
  author={Chung, Fan},
  journal={Annals of Combinatorics},
  volume={9},
  number={1},
  pages={1--19},
  year={2005},
  publisher={Springer}
}

@book{ma2021deep,
  title={Deep learning on graphs},
  author={Ma, Yao and Tang, Jiliang},
  year={2021},
  publisher={Cambridge University Press}
}

@article{greenacre2021compositional,
  title={Compositional data analysis},
  author={Greenacre, Michael},
  journal={Annual Review of Statistics and its Application},
  volume={8},
  number={1},
  pages={271--299},
  year={2021},
  publisher={Annual Reviews}
}

@inproceedings{brown2020language,
  title={Language Models are Few-Shot Learners},
  author={Brown, Tom and Mann, Benjamin and Ryder, Nick and Subbiah, Melanie and Kaplan, Jared and Dhariwal, Prafulla and Neelakantan, Arvind and Shyam, Pranav and Sastry, Girish and Askell, Amanda and others},
  booktitle={Advances in Neural Information Processing Systems (NeurIPS)},
  year={2020}
}

@inproceedings{wei2022chain,
  title={Chain-of-thought prompting elicits reasoning in large language models},
  author={Wei, Jason and Wang, Xuezhi and Schuurmans, Dale and Bosma, Maarten and Ichter, Brian and Xia, Fei and Chi, Ed and Le, Quoc V. and Zhou, Denny and others},
  booktitle={Advances in Neural Information Processing Systems (NeurIPS)},
  year={2022}
}

@article{achiam2023gpt,
  title={Gpt-4 technical report},
  author={Achiam, Josh and Adler, Steven and Agarwal, Sandhini and Ahmad, Lama and Akkaya, Ilge and Aleman, Florencia Leoni and Almeida, Diogo and Altenschmidt, Janko and Altman, Sam and Anadkat, Shyamal and others},
  journal={arXiv preprint arXiv:2303.08774},
  year={2023}
}

@article{gemini2023,
  title   = {Gemini: A Family of Highly Capable Multimodal Models},
  author  = {Google DeepMind},
  journal = {arXiv preprint arXiv:2312.11805},
  year    = {2023}
}

@article{kojima2022large,
  title={Large Language Models are Zero-Shot Reasoners},
  author={Kojima, Takeshi and Gu, Shixiang Shane and Reid, Machel and Matsuo, Yutaka and Iwasawa, Yusuke},
  journal={arXiv preprint arXiv:2205.11916},
  year={2022}
}

@article{velivckovic2021neural,
  title={Neural algorithmic reasoning},
  author={Veli{\v{c}}kovi{\'c}, Petar and Blundell, Charles},
  journal={Patterns},
  volume={2},
  number={7},
  year={2021},
  publisher={Elsevier}
}

@article{mcculloch1943logical,
  title={A Logical Calculus of the Ideas Immanent in Nervous Activity},
  author={McCulloch, Warren S. and Pitts, Walter},
  journal={The Bulletin of Mathematical Biophysics},
  volume={5},
  pages={115--133},
  year={1943},
  doi={10.1007/BF02478259}
}

@book{muroga1971threshold,
  title={Threshold Logic and Its Applications},
  author={Muroga, Saburo},
  publisher={Wiley-Interscience},
  address={New York},
  year={1971},
  isbn={0471625302}
}

@book{parberry1994circuit,
  title={Circuit Complexity and Neural Networks},
  author={Parberry, Ian},
  publisher={MIT Press},
  year={1994},
  isbn={9780262161480}
}

@inproceedings{maass1991sigmoid,
  title={On the Computational Power of Sigmoid versus Boolean Threshold Circuits},
  author={Maass, Wolfgang and Schnitger, Georg and Sontag, Eduardo D.},
  booktitle={Proceedings of the 32nd Annual IEEE Symposium on Foundations of Computer Science},
  pages={767--776},
  year={1991},
  doi={10.1109/SFCS.1991.185447}
}

@article{merrill2022saturated,
  title={Saturated Transformers are Constant-Depth Threshold Circuits},
  author={Merrill, William and Sabharwal, Ashish and Smith, Noah A.},
  journal={Transactions of the Association for Computational Linguistics},
  volume={10},
  pages={843--856},
  year={2022},
  doi={10.1162/tacl_a_00493}
}

@inproceedings{velickovic2020neural,
  title={Neural Execution of Graph Algorithms},
  author={Veli{\v{c}}kovi{\'c}, Petar and Ying, Rex and Padovano, Matilde and Hadsell, Raia and Blundell, Charles},
  booktitle={International Conference on Learning Representations},
  year={2020}
}

@inproceedings{velickovic2022clrs,
  title={The {CLRS} Algorithmic Reasoning Benchmark},
  author={Veli{\v{c}}kovi{\'c}, Petar and Badia, Adri{\`a} Puigdom{\`e}nech and Budden, David and Pascanu, Razvan and Banino, Andrea and Dashevskiy, Misha and Hadsell, Raia and Blundell, Charles},
  booktitle={Proceedings of the 39th International Conference on Machine Learning},
  pages={22084--22102},
  year={2022},
  publisher={PMLR}
}

@inproceedings{ibarz2022generalist,
  title={A Generalist Neural Algorithmic Learner},
  author={Ibarz, Borja and Kurin, Vitaly and Papamakarios, George and Nikiforou, Kyriacos and Bennani, Mehdi and Csord{\'a}s, R{\'o}bert and Dudzik, Andrew and Bo{\v{s}}njak, Matko and Vitvitskyi, Alex and Rubanova, Yulia and Deac, Andreea and Bevilacqua, Beatrice and Ganin, Yaroslav and Blundell, Charles and Veli{\v{c}}kovi{\'c}, Petar},
  booktitle={Learning on Graphs Conference},
  year={2022},
  publisher={PMLR}
}

@inproceedings{selsam2018learning,
title={Learning a {SAT} Solver from Single-Bit Supervision},
author={Daniel Selsam and Matthew Lamm and Benedikt B\"{u}nz and Percy Liang and Leonardo de Moura and David L. Dill},
booktitle={International Conference on Learning Representations},
year={2019},
url={https://openreview.net/forum?id=HJMC_iA5tm},
}

@article{DoshiVelezKim2017Interpretability,
  title={Towards a Rigorous Science of Interpretable Machine Learning},
  author={Doshi-Velez, Finale and Kim, Been},
  journal={arXiv preprint arXiv:1702.08608},
  year={2017}
}

@article{Lipton2018Mythos,
  title={The Mythos of Model Interpretability},
  author={Lipton, Zachary C.},
  journal={Communications of the ACM},
  volume={61},
  number={10},
  pages={36--43},
  year={2018}
}

@article{OlahEtAl2020ZoomIn,
  title={Zoom In: An Introduction to Circuits},
  author={Olah, Chris and Cammarata, Nick and Schubert, Ludwig and Goh, Gabriel and Petrov, Michael and Carter, Shan},
  journal={Distill},
  year={2020}
}

@article{ElhageEtAl2021TransformerCircuits,
  title={A Mathematical Framework for Transformer Circuits},
  author={Elhage, Nelson and Nanda, Neel and Olsson, Catherine and Henighan, Tom and Joseph, Nicholas and Mann, Ben and Askell, Amanda and Bai, Yuntao and Chen, Anna and Conerly, Tom and DasSarma, Nova and Drain, Dawn and Ganguli, Deep and Hatfield-Dodds, Zac and Hernandez, Danny and Jones, Andy and Kernion, Jackson and Lovitt, Liane and Ndousse, Kamal and Amodei, Dario and Brown, Tom and Clark, Jack and Kaplan, Jared and McCandlish, Sam and Olah, Chris},
  journal={Transformer Circuits Thread},
  year={2021}
}

@inproceedings{MitchellEtAl2019ModelCards,
  title={Model Cards for Model Reporting},
  author={Mitchell, Margaret and Wu, Simone and Zaldivar, Andrew and Barnes, Parker and Vasserman, Lucy and Hutchinson, Ben and Spitzer, Elena and Raji, Inioluwa Deborah and Gebru, Timnit},
  booktitle={Proceedings of the Conference on Fairness, Accountability, and Transparency},
  pages={220--229},
  year={2019}
}

@article{GebruEtAl2021Datasheets,
  title={Datasheets for Datasets},
  author={Gebru, Timnit and Morgenstern, Jamie and Vecchione, Briana and Vaughan, Jennifer Wortman and Wallach, Hanna and Daum{\'e} III, Hal and Crawford, Kate},
  journal={Communications of the ACM},
  volume={64},
  number={12},
  pages={86--92},
  year={2021}
}

@inproceedings{BuolamwiniGebru2018GenderShades,
  title={Gender Shades: Intersectional Accuracy Disparities in Commercial Gender Classification},
  author={Buolamwini, Joy and Gebru, Timnit},
  booktitle={Proceedings of Machine Learning Research},
  volume={81},
  pages={77--91},
  year={2018}
}

@inproceedings{RajiEtAl2020AccountabilityGap,
  title={Closing the AI Accountability Gap: Defining an End-to-End Framework for Internal Algorithmic Auditing},
  author={Raji, Inioluwa Deborah and Smart, Andrew and White, Rebecca N. and Mitchell, Margaret and Gebru, Timnit and Hutchinson, Ben and Smith-Loud, Jamila and Theron, Daniel and Barnes, Parker},
  booktitle={Proceedings of the 2020 Conference on Fairness, Accountability, and Transparency},
  pages={33--44},
  year={2020}
}

\appendix \label{appendix}

\section{Supplementary preliminaries on Wasserstein concentration} 
\label{appx:prelim}

Let $(\mathscr{X},d_{\mathscr{X}})$ be a metric space, and let $\mathrm{diam}(\mathscr{X}) \eqdef \sup_{x,y} d_{\mathscr{X}}(x,y)$ be its diameter.
For $\alpha\in (0,1]$, we regard $(\mathscr{X}, d_{\mathscr{X}}^{\alpha})$ as the $\alpha$-snowflaked version of $(\mathscr{X}, d_{\mathscr{X}})$, where $d_{\mathscr{X}}^{\alpha}(x,y) \eqdef d_{\mathscr{X}}(x,y)^{\alpha}$.
The $\alpha$-H\"older Wasserstein distance between two probability measures $\lambda$, $\lambda'$ on $\mathscr{X}$ 
is given by (see \citep[Definition 9]{hou2023instance})\footnote{Definition \eqref{eqdef:Wassdistance} is inspired by the fact that setting $\alpha=1$ recovers the dual definition \citep[Remark 6.5]{villani2009optimal} of the Wasserstein $\mathcal{W}_1$ transport distance \citep[Definition~6.1]{villani2009optimal}.} 
\begin{equation} \label{eqdef:Wassdistance}
    \mathcal{W}_{\alpha}(\lambda,\lambda')
    \eqdef 
    \sup_{f\in \mathrm{H}(\alpha,\mathscr{X},1)}\, 
    \mathbb{E}_{X\sim \lambda}[f(X)] - \mathbb{E}_{Y\sim \lambda'}[f(Y)],
\end{equation}
where, for $\mathtt{C}\geq 0$, $\mathrm{H}(\alpha,\mathscr{X},\mathtt{C})$ denotes the set of real-valued $\alpha$-H\"older continuous functions $f$ on $\mathscr{X}$ satisfying 
\begin{equation} \label{stapleton}
    |f(x)-f(y)|\leq \mathtt{C}d_{\mathscr{X}}(x,y)^{\alpha},
\end{equation}
for every $x,y\in\mathscr{X}$.
If $\alpha=1$ then $\mathrm{H}(1,\mathscr{X},\mathtt{C})=\mathrm{Lip}(\mathscr{X},\mathtt{C})$, the set of real-valued $\mathtt{C}$-Lipschitz continuous functions on $\mathscr{X}$. 
Note that definition \eqref{stapleton}, and thus \eqref{eqdef:Wassdistance}, depends on the metric choice. For example, if $\mathscr{X}$ have been equipped with $d_{\mathscr{X}}^{\alpha}$, then \eqref{stapleton} would describe a $\mathtt{C}$-Lipschitz function on $(\mathscr{X},d_{\mathscr{X}}^{\alpha})$. 

The following lemma is the main result of this appendix section and provides the key ingredient in the proof of Theorem~\ref{thrm:main_result}.

\begin{lemma}
\label{lem:New_Convergence__SuperAssouad}
Let $\alpha\in (0,1)$.
Let $(\mathscr{X},d_{\mathscr{X}})$ be a $k$-point metric space, with $k\geq 2$. 
Suppose $1\leq d_{\mathscr{X}}(x,y)$ for every $x\not=y\in\mathscr{X}$.
Let $\lambda$ be a probability measure on $\mathscr{X}$, and let $\lambda^N$ be its associated empirical measure.
Then the following holds for every $\delta\in (0,1)$ and $N\in\N$,
\begin{equation} \label{precious}
    \mathcal{W}_{\alpha}(\lambda,\lambda^N) 
    \lesssim
    \mathrm{diam}(\mathscr{X})^{3\alpha/2} \cdot
    \Big(\frac{1}{\sqrt{N}} + \frac{\sqrt{\log (2/\delta)}}{\sqrt{N}}\Big)
\end{equation}
with probability at least $1-\delta$. 
\end{lemma}

The proof of Lemma~\ref{lem:New_Convergence__SuperAssouad} relies on two main steps. 
First, we map $(\mathscr{X},d_{\mathscr{X}}^{\alpha})$ bi-Lipschitz-ly to $(\mathbb{R},d_{\infty})$, where $d_{\infty}$ denotes the $\ell^{\infty}$-metric. Second, we apply a known result for the Wasserstein distance $\mathcal{W}_1$ in $(\mathbb{R},d_{\infty})$. These steps are developed in Propositions~\ref{prop:independent embedding} and~\ref{prop:con_holder_wass} below.

\begin{proposition}[Low distortion snowflake-embedding into $\mathbb{R}$]
\label{prop:independent embedding}
Let $\alpha\in (0,1)$.
Let $(\mathscr{X}, d_{\mathscr{X}})$ be a $k$-point metric space with $k\geq 2$. 
Suppose $1\leq d_{\mathscr{X}}(x,y)\leq D$ for $x\not=y\in\mathscr{X}$.
Then there exists a threshold $\theta_{*} \in (0,1)$ depending on $\alpha$ and $\mathscr{X}$ such that the following holds. 
For every $\theta\in (\theta_{*},1)$, there exists a bi-Lipschitz embedding $\varphi_{\alpha,\theta}: (\mathscr{X}, d_{\mathscr{X}}^{\alpha}) \to (\mathbb{R},d_{\infty})$ satisfying
\begin{align} \label{shrinkstretch}
    \nonumber (1+1/20)^{-3/2}(1+\theta)^{-1} D^{-\alpha/2} d_{\mathscr{X}}(x,y)^{\alpha} &\leq |\varphi_{\alpha,\theta}(x) - \varphi_{\alpha,\theta}(y)|
    \\
    &\leq (1+1/20)^{3/2}(1+\theta) d_{\mathscr{X}}(x,y)^{\alpha}.
\end{align}
Particularly, $\mathrm{diam}(\varphi_{\alpha,\theta}(\mathscr{X}))\leq (1+1/20)^{3/2}(1+\theta) D^{\alpha}$.
\end{proposition}

\begin{proposition} \label{prop:con_holder_wass}
Let $\mathscr{X}$ be a compact subset of $\R$. 
Let $\lambda$ be a probability measure on $\mathscr{X}$, and let $\lambda^N$ be its empirical measure. 
Then for all $t > 0$ and all $N\ge 4$, 
\begin{equation*}
    \mathbb{P}\Big(\big|\mathcal{W}_1(\lambda,\lambda^N) - \mathbb{E}[\mathcal{W}_1(\lambda,\lambda^N)] \big| \geq t \Big)  \leq 2e^{-\frac{2Nt^2}{\mathrm{diam}(\mathscr{X})^2}},
\end{equation*}
and 
\begin{equation} \label{eq:criticaldiv}
    \mathbb{E}[\mathcal{W}_1(\lambda,\lambda^N)] \leq \frac{\sqrt{2}\mathrm{diam}(\mathscr{X})}{\sqrt{N}}.
\end{equation}
\end{proposition}

Proposition~\ref{prop:con_holder_wass} is a key technical variant of \citep[Lemma~16]{hou2023instance}, adapted to the setting of $(\mathbb{R},d_{\infty})$, and is established in \citep[Lemma~B.2]{detering2025learning}.
Proposition~\ref{prop:independent embedding} presents a bi-Lipschitz embedding result of potential independent interest for $(\mathscr{X}, d_{\mathscr{X}}^{\alpha})$ into $(\mathbb{R},d_{\infty})$, whose proof is given below. 
The proof of Lemma~\ref{lem:New_Convergence__SuperAssouad} then follows immediately.\\

\begin{proof}[Proof of Proposition~\ref{prop:independent embedding}]
Assume first $\alpha\not=1/2$.
We construct a bi-Lipschitz embedding $\varphi_{\alpha,\theta}: (\mathscr{X},d_{\mathscr{X}}^{\alpha}) \to (\mathbb{R},d_{\infty})$ via the composition of the following bi-Lipschitz maps:
\begin{equation} \label{dream}
    (\mathscr{X},d_{\mathscr{X}}^{\alpha}) \to (\mathbb{R}^{m_{\alpha}}, d_{\infty}) \to (\mathbb{R}^{m_{\alpha}}, d_{\infty}^{1/2}) \to (\mathbb{R},d_{\infty}),
\end{equation}
where $m_{\alpha}$ is to be determined.
Some of the results invoked below are stated in terms of the \emph{doubling constant}\footnote{$(\mathscr{X},d_{\mathscr{X}})$ is doubling with the doubling constant $\mathtt{M}\in\mathbb{N}$, if for every $r\geq 0$ and every $x\in\mathscr{X}$, the closed ball $B(x,r)\eqdef \{y\in\mathscr{X}: d_{\mathscr{X}}(x,y)\leq r\}$ can be covered by some $\mathtt{M}$ closed balls $B(x_1,r/2),\dots,B(x_{\mathtt{M}},r/2)$, i.e. $B(x,r) \subset \bigcup_{i=1}^{\mathtt{M}} B(x_i,r/2)$, and if $\mathtt{M}$ is the smallest such number.}. 
Since we work exclusively with $k$-point metric spaces with $k\geq 2$, we only require the fact that their doubling constants are at least $2$ and at most $k$ (see e.g. \citep[Example~2.2]{detering2025learning}).
For the first embedding map in \eqref{dream}, we appeal to the $\ell^{\infty}$ version of Assouad's Embedding Theorem due to~\citep[Theorem 3]{OferLowDimEmbeddingDoublingSpaces}. 
This result ensures that for every $\alpha\in (0,1)$, there exists a bi-Lipschitz embedding $\psi_{\alpha}:(\mathscr{X},d_{\mathscr{X}}^{\alpha})\to (\mathbb{R}^{m_{\alpha}},d_{\infty})$ with distortion at most $1+1/20$ and
\begin{equation} \label{eq:firstrange}
    m_{\alpha} = \bigg\lceil \frac{\mathtt{M}^{6+\log_2 20}}{\alpha(1-\alpha)}\bigg\rceil,
\end{equation}
where $\mathtt{M}\geq 2$ denotes the doubling constant of $(\mathscr{X},d_{\mathscr{X}})$.
The exponent of $\mathtt{M}$ in \eqref{eq:firstrange} can be deduced from the proof of~\citep[Theorem 3]{OferLowDimEmbeddingDoublingSpaces} together with~\citep[Proposition 2]{OferLowDimEmbeddingDoublingSpaces}.
In particular, under either the upper- or lower-normalized distortion convention\footnote{In the upper-normalized convention, $(1+\theta)^{-1}d(x,y) \leq d(f(x),f(y))\leq d(x,y)$, while in the lower-normalized convention, $d(x,y) \leq d(f(x),f(y))\leq (1+\theta) d(x,y)$, where $1+\theta$ is a given distortion.}, we obtain
\begin{equation} \label{hoola}
    (1+1/20)^{-1} d_{\mathscr{X}}(x,y)^{\alpha}\leq \|\psi_{\alpha}(x) - \psi_{\alpha}(y)\|_{\infty} \leq (1+1/20) d_{\mathscr{X}}(x,y)^{\alpha},
\end{equation}
and therefore,
\begin{equation} \label{diamexp1}
    \mathrm{diam}(\psi_{\alpha}(\mathscr{X})) \leq (1+1/20) \mathrm{diam}(\mathscr{X})^{\alpha}.
\end{equation}
Turning to the second embedding in \eqref{dream}, we consider the identity map $\iota: (\psi_{\alpha}(\mathscr{X}), d_{\infty}) \to (\psi_{\alpha}(\mathscr{X}), d_{\infty}^{1/2})$. 
Since $(\psi_{\alpha}(\mathscr{X}), d_{\infty})$ is a $k$-point metric space, this map is automatically bi-Lipschitz.
Indeed, we can estimate its bi-Lipschitz constants as follows.
From \eqref{hoola} and that $1\leq d_{\mathscr{X}}(x,y) \leq D$ for $x\not=y\in\mathscr{X}$,
\begin{equation*}
     (1+1/20)^{-1} D^{-\alpha} \leq \|\psi_{\alpha}(x) - \psi_{\alpha}(y)\|_{\infty}^{-1} \leq (1+1/20).
\end{equation*}
Thus,
\begin{align} \label{midLip}
    \nonumber (1+1/20)^{-1/2} D^{-\alpha/2} \|\psi_{\alpha}(x) - \psi_{\alpha}(y)\|_{\infty} &\leq \|\iota\circ \psi_{\alpha}(x) - \iota\circ\psi_{\alpha}(y)\|_{\infty}^{1/2} \\
    \nonumber &= \|\psi_{\alpha}(x) - \psi_{\alpha}(y)\|_{\infty} \|\psi_{\alpha}(x) - \psi_{\alpha}(y)\|_{\infty}^{-1/2} \\ 
    &\leq (1+1/20)^{1/2}\|\psi_{\alpha}(x) - \psi_{\alpha}(y)\|_{\infty}.
\end{align}
Moreover, \eqref{diamexp1}, \eqref{midLip} imply
\begin{equation} \label{diamexp2}
    \mathrm{diam}(\iota\circ\psi_{\alpha}(\mathscr{X})) \leq (1+1/20)^{3/2} D^{\alpha}.
\end{equation}
Let $\mathtt{M}'$ be the doubling constant of $(\psi_{\alpha}(\mathscr{X}), d_{\infty})$. Although it can be verified from \eqref{hoola}, \eqref{midLip} that $\mathtt{M}'$ is at most a constant power of $\mathtt{M}$, we only need that $\mathtt{M}'\geq 2$, which holds since $(\psi_{\alpha}(\mathscr{X}), d_{\infty})$ is a $k$-point metric space with $k\geq 2$.
Continuing to the last embedding in \eqref{dream}, we invoke~\citep[Theorem 6.6]{HarPeledMender_2006EmbeddingViaFastNets_SIAMJCompute} and its proof, which implies that for every $\theta\in (0,1)$\footnote{The fourth paragraph of the proof of~\citep[Theorem 6.6]{HarPeledMender_2006EmbeddingViaFastNets_SIAMJCompute} implicitly assumes that the distortion $1+\theta$ lies in $(1,2)$.}, there exists a bi-Lipschitz embedding $\psi_{\theta}:(\iota\circ\psi_{\alpha}(\mathscr{X}), d_{\infty}^{1/2})\to (\mathbb{R}^{m_{\theta}},d_{\infty})$ with distortion at most $1+\theta$, where
\begin{equation} \label{desiredim}
    1 \le m_{\theta} \le \lfloor\theta^{-C^*\log_2(\mathtt{M}')}\rfloor,
\end{equation}
and $C^*> 1$ is an absolute constant\footnote{This follows from the proof of \citep[Theorem 6.6]{HarPeledMender_2006EmbeddingViaFastNets_SIAMJCompute}.}. 
Let $\theta'_{*} \eqdef 1/2^{1/C^*\log_2(\mathtt{M}')}$. We see that if $\theta\in (\theta'_{*},1)$,
\begin{equation} \label{thresholdM}
    \theta^{C^*\log_2(\mathtt{M}')} > \Big(\frac{1}{2^{1/C^*\log_2(\mathtt{M}')}}\Big)^{C^*\log_2(\mathtt{M}')} = \frac{1}{2},
\end{equation}
since $\mathtt{M}'\geq 2$.
Thus, for $\theta$ in this range, \eqref{desiredim} implies that $m_{\theta}=\lfloor\theta^{-C^*\log_2(\mathtt{M}')}\rfloor=1$ and yields a bi-Lipschitz embedding $\psi_{\theta}: (\iota\circ\psi_{\alpha}(\mathscr{X}), d_{\infty}^{1/2}) \to (\mathbb{R},d_{\infty})$ that is
\begin{align} \label{lastLip}
    \nonumber (1+\theta)^{-1} \|\iota\circ\psi_{\alpha}(x) - \iota\circ\psi_{\alpha}(y)\|_{\infty}^{1/2} &\leq |\psi_{\theta}\circ\iota\circ\psi_{\alpha}(x) - \psi_{\theta}\circ\iota\circ\psi_{\alpha}(y)| \\
    &\leq (1+\theta) \|\iota\circ\psi_{\alpha}(x) - \iota\circ\psi_{\alpha}(y)\|_{\infty}^{1/2}.
\end{align}
Let $\varphi_{\alpha,\theta} \eqdef \psi_{\theta}\circ\iota\circ\psi_{\alpha}$.
Then combining \eqref{hoola}, \eqref{midLip}, \eqref{lastLip}, we obtain \eqref{shrinkstretch}, while combining \eqref{diamexp2}, \eqref{lastLip} yields the final claim on the diameter.

Assume now $\alpha=1/2$. We directly embed $(\mathscr{X},d_{\mathscr{X}}^{1/2})$ into $(\mathbb{R},d_{\infty})$ by adapting the last embedding in \eqref{dream}. 
In particular, we let $\theta_{*} \eqdef 1/2^{1/C^*\log_2(\mathtt{M})}$, where $\mathtt{M}\geq 2$ is the doubling constant of $(\mathscr{X},d_{\mathscr{X}})$. Arguing as in \eqref{desiredim}, \eqref{thresholdM}, if $\theta\in (\theta_{*},1)$, then $m_{\theta}=\lfloor\theta^{-C^*\log_2(\mathtt{M})}\rfloor=1$. 
This yields, by \citep[Theorem~6.6]{HarPeledMender_2006EmbeddingViaFastNets_SIAMJCompute}, a bi-Lipschitz embedding $\varphi_{1/2,\theta} \eqdef \psi_{\theta}: (\mathscr{X},d_{\mathscr{X}}^{1/2})\to (\mathbb{R},d_{\infty})$.
\end{proof}

\begin{proof}[Proof of Lemma~\ref{lem:New_Convergence__SuperAssouad}]
The proof is adapted from the argument in \citep[Proposition~4.1]{detering2025learning}.
We recall Proposition~\ref{prop:independent embedding}, which states that for each $\theta\in (\theta_{*},1)$, there exist a bi-Lipschitz embedding $\varphi_{\alpha,\theta}: (\mathscr{X}, d_{\mathscr{X}}^{\alpha}) \to (\mathbb{R}, d_{\infty})$, such that
\begin{equation} \label{eq:diamstretch}
    \mathrm{diam}(\varphi_{\alpha,\theta}(\mathscr{X}))\leq \mathtt{S} \mathrm{diam}(\mathscr{X})^{\alpha},
\end{equation}
where $\mathtt{S}\eqdef (1+1/20)^{3/2}(1+\theta)$.
We also set $\mathtt{R} \eqdef (1+1/20)^{-3/2}(1+\theta)^{-1} \mathrm{diam}(\mathscr{X})^{-\alpha/2}$.
Define $\gamma\eqdef (\varphi_{\alpha,\theta})_{\#}(\lambda)$ and $\gamma^N\eqdef (\varphi_{\alpha,\theta})_{\#}\lambda^N$, which are probability measures on the $k$-point metric space $(\varphi_{\alpha,\theta}(\mathscr{X}),d_{\infty}) \subset (\mathbb{R},d_{\infty})$. 
By combining \eqref{eq:diamstretch} with Proposition~\ref{prop:con_holder_wass}, we obtain
\begin{equation} \label{eq:BenoitConcentration2}
    \mathbb{E}[\mathcal{W}_1(\gamma,\gamma^N)]
    \le  
    \frac{\sqrt{2} \mathrm{diam}(\varphi_{\alpha,\theta}(\mathscr{X}))}{\sqrt{N}} \leq \frac{\sqrt{2}\mathtt{S}\mathrm{diam}(\mathscr{X})^{\alpha}}{\sqrt{N}}
\end{equation}
and for each $t>0$,
\begin{equation} \label{eq:BenoitConcentration1}
    \mathbb{P}\Big(\big|\mathcal{W}_1(\gamma,\gamma^N) - \mathbb{E}[\mathcal{W}_1(\gamma,\gamma^N)]\big|
    \geq t \Big)
    \le 
    2 e^{-\frac{2N t^2}{\mathtt{S}^2\mathrm{diam}(\mathscr{X})^{2\alpha}}} \leq 2 e^{-\frac{N t^2}{16\mathrm{diam}(\mathscr{X})^{2\alpha}}}.
\end{equation}
We translate \eqref{eq:BenoitConcentration2}, \eqref{eq:BenoitConcentration1} into expressions of Wasserstein distances between $\lambda$, $\lambda^N$ as follows.
By Proposition~\ref{prop:independent embedding}, $\varphi_{\alpha,\theta}$ is $\mathtt{S}$-Lipschitz on $(\mathscr{X}, d_{\mathscr{X}}^{\alpha})$, and its inverse $\varphi_{\alpha,\theta}^{-1}$ is $\mathtt{R}^{-1}$-Lipschitz on $(\varphi_{\alpha,\theta}(\mathscr{X}),d_{\infty})$. 
It follows that, if $f\in \mathrm{H}(\alpha,\mathscr{X},1)$ \eqref{stapleton}, then $f\circ \varphi_{\alpha,\theta}^{-1} \in \mathrm{H}(1,\varphi_{\alpha,\theta}(\mathscr{X}),\mathtt{R}^{-1})$,
and conversely, if $f\circ \varphi_{\alpha,\theta}^{-1}\in \mathrm{H}(1,\varphi_{\alpha,\theta}(\mathscr{X}),1)$, then $f \in \mathrm{H}(\alpha,\mathscr{X},\mathtt{S})$.
Indeed, for $x,y\in \varphi_{\alpha,\theta}(\mathscr{X})$,
\begin{equation*}
    |f\circ \varphi_{\alpha,\theta}^{-1}(x) - f\circ \varphi_{\alpha,\theta}^{-1}(y)|
    \le 
    d_{\mathscr{X}}(\varphi_{\alpha,\theta}^{-1}(x), \varphi_{\alpha,\theta}^{-1}(y))^{\alpha}\\
    \le \mathtt{R}^{-1}| x-y |,
\end{equation*}
and for $x,y\in \mathscr{X}$,
\begin{equation*}
    |f(x) - f(y)| = |f\circ \varphi_{\alpha,\theta}^{-1}( \varphi_{\alpha,\theta}(x)) - f\circ \varphi_{\alpha,\theta}^{-1}(\varphi_{\alpha,\theta}(y))|
    \le 
    |\varphi_{\alpha,\theta}(x) - \varphi_{\alpha,\theta}(y)| \\
    \le \mathtt{S} d_{\mathscr{X}}(x,y)^{\alpha}.
\end{equation*}
Therefore, by a change of variables, we get
\begin{equation} \label{eq:Wassineqs}
    \mathcal{W}_{\alpha}(\lambda,\lambda^N) \le \mathtt{R}^{-1}\mathcal{W}_1(\gamma,\gamma^N) \quad\text{ and }\quad
    \mathcal{W}_1(\gamma,\gamma^N) \le \mathtt{S}\mathcal{W}_{\alpha}(\lambda,\lambda^N).
\end{equation}
By definition, $\mathtt{S}, \mathtt{S}^{-1}\asymp 1$, and $\mathtt{S}, \mathtt{S}^{-1}\leq \mathtt{R}^{-1}$, and $\mathtt{R}^{-1}\asymp \mathrm{diam}(\mathscr{X})^{\alpha/2}$.
Thus, on the one hand, combining \eqref{eq:BenoitConcentration2}, \eqref{eq:Wassineqs} yields
\begin{equation} \label{preforwardexpectation}
    \mathbb{E}[\mathcal{W}_{\alpha}(\lambda,\lambda^N)] 
    \leq 
    \frac{C\mathrm{diam}(\mathscr{X})^{3\alpha/2}}{\sqrt{N}},
\end{equation}
for some absolute $C>0$.
On the other hand, combining \eqref{eq:BenoitConcentration2}, \eqref{eq:BenoitConcentration1}, \eqref{eq:Wassineqs} allows us to derive, for $t>0$, 
\begin{align} \label{eq:forwardconcentrationgen1}
    \nonumber \mathcal{W}_{\alpha} (\lambda,\lambda^N) - \mathbb{E}[\mathcal{W}_{\alpha}(\lambda,\lambda^N)] 
    & \le
    \mathtt{R}^{-1}\mathcal{W}_1(\gamma,\gamma^N) -
    \mathtt{S}^{-1} \mathbb{E} [\mathcal{W}_1(\gamma,\gamma^N)]\\
    \nonumber & \le
    (\mathtt{R}^{-1}-\mathtt{S}^{-1})\mathbb{E} [\mathcal{W}_1(\gamma,\gamma^N)]
    + \mathtt{R}^{-1}t \\
    &\leq \frac{C\mathrm{diam}(\mathscr{X})^{3\alpha/2}}{\sqrt{N}} + \mathrm{diam}(\mathscr{X})^{\alpha/2} t,
\end{align}
for some absolute $C>0$, which happens with probability at least $1- 2e^{-N t^2/(16\mathrm{diam}(\mathscr{X})^{2\alpha}})$.
Writing $\mathcal{W}_{\alpha} (\lambda,\lambda^N) = \mathbb{E}[\mathcal{W}_{\alpha}(\lambda,\lambda^N)] + \mathcal{W}_{\alpha} (\lambda,\lambda^N) - \mathbb{E}[\mathcal{W}_{\alpha}(\lambda,\lambda^N)]$ and combining with \eqref{preforwardexpectation}, \eqref{eq:forwardconcentrationgen1}, we obtain the desired conclusion.
\end{proof}

\section{Proof of main theorem} \label{appx:main}

\begin{proof}[Proof of Theorem~\ref{thrm:main_result}]
The proof proceeds along the lines of \citep[Theorem~3.1]{detering2025learning}.
Define a metric on $G_\mathrm{sc}\times \Delta^{\circ}_{\m}$ by
\begin{equation*}
    d_{G_\mathrm{sc}\times \Delta^{\circ}_{\m}}((v,\mathbf{p}),(w,\mathbf{q})) \eqdef \max\{d_{G_\mathrm{sc}}(v,w), d_A(\mathbf{p},\mathbf{q})\},
\end{equation*}
where $d_{G_\mathrm{sc}}$ denotes the hitting probability metric \eqref{hittingmetric} on $G_\mathrm{sc}$ and $d_A$ the Aitchison metric \eqref{Ametric}.
Let $\mathscr{D} \eqdef \{(v,Q_{\eta}(v)): v\in \Gamma\}$ inherit the induced metric.
For $h\in\mathcal{H}$, we associate $J_h: \Gamma\times \Delta^{\circ}_{\m} \to\mathbb{R}_{\geq 0}$, defined to be $J_h(v,\mathbf{p}) \eqdef J(\pi_v\circ h(\mathbf{x}), \mathbf{p})^{1/2}$.
Then $J_h|_\mathscr{D}$ is a function of $v\in \Gamma$---indeed,
\begin{equation*} 
    (J_h|_\mathscr{D})(v,\mathbf{p}) = J_h(v,Q_{\eta}(v))= J(\pi_v\circ h(\mathbf{x}),Q_{\eta}(v))^{1/2}.
\end{equation*}
By recalling \eqref{eqdef:emprisk}, \eqref{eqdef:truerisk} and that $\mu = (\mathbbm{1}\times Q_{\eta})_{\#}\mu_{[k]}$ \eqref{iidrv} and $\mu^N$ are both supported on $\Gamma$, we interpret
\begin{equation*}
    \mathcal{R}_{\mathbf{x},t}^{\alpha}(h) = \mathbb{E}_{(V,\mathbf{P})\sim\mu} [J_h(V,\mathbf{P})] \quad\text{ and }\quad \mathcal{R}_{\mathbf{x},t}^{\alpha, N}(h) = \mathbb{E}_{(V,\mathbf{P})\sim\mu^N} [J_h(V,\mathbf{P})].
\end{equation*}
Observe the following. Suppose $J_h|_{\mathscr{D}}$ is $\mathtt{C}_J^{\star}$-Lipschitz, for $\mathtt{C}_J^{\star}>0$, i.e.
\begin{equation} \label{eq:perculiarLip}
    |J(\pi_v\circ h(\mathbf{x}), Q_{\eta}(v)) - J(\pi_w\circ h(\mathbf{x}), Q_{\eta}(w))| \leq \mathtt{C}_{J}^{\star} d_{G_\mathrm{sc}}(v,w).
\end{equation}
Then by invoking Kantorovich-Rubinstein duality \citep[Remark 6.5 and Theorem 5.10(i)]{villani2009optimal}, we obtain
\begin{equation} \label{eq:RxW1/2}
    |\mathcal{R}_{\mathbf{x},t}^{\alpha}(h) - \mathcal{R}_{\mathbf{x},t}^{\alpha, N}(h)|
    \leq (\mathtt{C}_{J}^{\star})^{\alpha}\mathcal{W}_{\alpha}(\mu,\mu^N),
\end{equation}
where the Wasserstein distance $\mathcal{W}_{\alpha}$ is given in \eqref{eqdef:Wassdistance}.
Applying Lemma~\ref{lem:New_Convergence__SuperAssouad} to the metric space $(\mathscr{D}, d_{G_\mathrm{sc}\times \Delta^{\circ}_{\m}}|_{\mathscr{D}})$, we deduce that for every $\delta\in (0,1)$,
\begin{equation}
\label{eq:prf_main__cases1}
    \mathcal{W}_{\alpha}(\mu,\mu^N) \lesssim \mathrm{diam}(\mathscr{D})^{3\alpha/2} \cdot
    \Big(\frac{1}{\sqrt{N}} + \frac{\sqrt{\log (2/\delta)}}{\sqrt{N}}\Big)
\end{equation}
holds with probability at least $1-\delta$.
Substituting \eqref{eq:prf_main__cases1} into \eqref{eq:RxW1/2} and taking the supremum over $h\in \mathcal{H}$ gives
\begin{equation} \label{genpunch}
    \sup_{h\in\mathcal{H}} |\mathcal{R}_{\mathbf{x},t}^{\alpha}(h) - \mathcal{R}_{\mathbf{x},t}^{\alpha, N}(h)|
    \lesssim (\mathtt{C}_J^{\star}\mathrm{diam}(\mathscr{D})^{3/2})^{\alpha} \cdot
    \Big(\frac{1}{\sqrt{N}} + \frac{\sqrt{\log (2/\delta)}}{\sqrt{N}}\Big),
\end{equation}
with the same probability.
Thus to complete the argument, we estimate $\mathrm{diam}(\mathscr{D})$ and $\mathtt{C}_J^{\star}$ in \eqref{genpunch}. 
For the latter, if for every $h\in\mathcal{H}$ and every $v,w\in\Gamma$, $\|\pi_v\circ h(\mathbf{x}) - \pi_w\circ h(\mathbf{x})\|_A \leq \mathtt{C}_{\mathcal{H}} d_{G_\mathrm{sc}}(v,w)$, and if $\|Q_{\eta}(v) - Q_{\eta}(w)\|_A \leq \mathtt{C}_{Q_{\eta}} d_{G_\mathrm{sc}}(v,w)$, for some $\mathtt{C}_{\mathcal{H}}, \mathtt{C}_{Q_{\eta}}>0$, then for $(v,Q_{\eta}(v)), (w,Q_{\eta}(w))\in \mathscr{D}$, 
\begin{align} \label{ellLipschitz1}
    \nonumber |J(\pi_v\circ h(\mathbf{x}),Q_{\eta}(w)) - J(\pi_w\circ h(\mathbf{x}), Q_{\eta}(w)))|
    &\leq \mathtt{C}_J \|\pi_v\circ h(\mathbf{x}) - \pi_w\circ h(\mathbf{x})\|_A \\
    &\leq \mathtt{C}_J \mathtt{C}_{\mathcal{H}} d_{G_\mathrm{sc}}(v,w),
\end{align}
and 
\begin{align} \label{ellLipschitz2}
    |J(\pi_v\circ h(\mathbf{x}),Q_{\eta}(v)) - J(\pi_v\circ h(\mathbf{x}),Q_{\eta}(w))| 
    \leq \mathtt{C}_J \|Q_{\eta}(v) - Q_{\eta}(w)\|_A \leq \mathtt{C}_J \mathtt{C}_{Q_{\eta}} d_{G_\mathrm{sc}}(v,w).
\end{align}
Combining \eqref{eq:perculiarLip}, \eqref{ellLipschitz1}, \eqref{ellLipschitz2}, together with the triangle inequality, we arrive at an upper bound: 
\begin{equation} \label{ubound}
    \mathtt{C}_J^{\star} \leq \mathtt{C}_J (\mathtt{C}_{\mathcal{H}}+\mathtt{C}_{Q_{\eta}}).
\end{equation}
It remains to estimate the quantities $\mathrm{diam}(\mathscr{D})$, $\mathtt{C}_{\mathcal{H}}$, $\mathtt{C}_{Q_{\eta}}$, which are provided in the following proposition.

\begin{proposition} \label{prop:est}
It holds that
\begin{enumerate}
    \item[(i)] $\mathrm{diam}(\mathscr{D}) \lesssim \max\{\nu^{\h}, K_{\h}\}$,
    \item[(ii)] $\mathtt{C}_{Q_{\eta}} \leq K_{\h}$,
    \item[(iii)] $\mathtt{C}_{\mathcal{H}} \leq 2(\mathtt{m}-1)^{1/2}\big(\frac{3+\nu}{2}\big)^{\mathrm{p} (L-1)} \prod_{l=1}^L \beta_l$.
\end{enumerate}
\end{proposition}

The proof of Proposition~\ref{prop:est} is given below.
An application of \eqref{ubound} and Proposition~\ref{prop:est} to \eqref{genpunch} then gives \eqref{mainres} as desired.
\end{proof}

\begin{proof}[Proof of Proposition~\ref{prop:est}]
For notational convenience, we write the $(v,w)$-entry of a matrix $A$ as $A(v,w)$ for all matrices under consideration below.\\

\begin{proof}[Proof of Proposition~\ref{prop:est}(i)]
Suppose
\begin{equation} \label{prehand}
    \mathrm{diam}(G_\mathrm{sc}) \lesssim \nu^\h,
\end{equation}
which yields $\mathrm{diam}(\Gamma)\leq \mathrm{diam}(G_\mathrm{sc})\lesssim \nu^{\h}$.
Since \eqref{eq:boundedcomp} holds and since $\nu\geq 2$, it follows that 
\begin{equation*}
    \mathrm{diam}(\mathscr{D})\leq \max\{\mathrm{diam}(\Gamma),K_{\h}\} \lesssim \max\{\nu^{\h}, K_{\h}\},
\end{equation*}
proving part (i).
Thus in what follows, we establish \eqref{prehand}.
Although one can verify that $\mathrm{diam}(\Gamma)\asymp \nu^\h$, see Remark~\ref{rem:trudim}, the computation of $\mathrm{diam}(G_\mathrm{sc})$ is more straightforward.

Recall the identification \eqref{convenient} of $V_\mathrm{sc}$ with the index set $[k]$, under which $v_1,v_2,\dots,v_{\nu^{\h}}$ denote the base nodes of the computation tree $\mathtt{B}_{\nu}$, with $v_i$ connected to the tape position $T_{i-1}$. 
We also recall our convention of mixed notation: nodes corresponding to tape positions are still referred to by $T_i$, and the tree root by $\mathtt{r}$.
We begin with an estimate of the Perron vector $\phi$ of the irreducible stochastic transition matrix $P_{G_\mathrm{sc}}$ \eqref{transition}. 
By stationary, $\phi^{\top} P_{G_\mathrm{sc}} = \phi^{\top}$, that is
\begin{equation} \label{station}
    \phi(v) = \sum_{w: (w,v)\in E_\mathrm{sc}} \phi(w) P_{G_\mathrm{sc}}(w,v).
\end{equation}
If all the nodes $w$ in \eqref{station} have only one outgoing edge $(w,v)$ in $G_\mathrm{sc}$---meaning, $P_{G_\mathrm{sc}}(w,v)=1$---then \eqref{station} reduces to
\begin{equation} \label{nonslit}
    \phi(v) = \sum_{w: (w,v)\in E_\mathrm{sc}} \phi(w).
\end{equation}
By construction, the nodes $v$ satisfying \eqref{nonslit} are precisely 
\begin{equation*}
    S \eqdef V_\mathrm{sc}\setminus\{T_1,\dots,T_{\nu^{\h}-1},v_1,\dots,v_{\nu^{\h}-1}\}= \Gamma\cup\{T_0\}.
\end{equation*}
In particular, inductive application of \eqref{nonslit} across the computation tree layers, together with the positivity of $\phi$, gives
\begin{equation} \label{trivial}
    \phi(v) \leq \phi(\mathtt{r}) \quad\text{ for }\quad v\in V_{\nu}.
\end{equation}
In fact, one can argue that $\phi(v)\leq\phi(\mathtt{r})$, for every $v\in G_\mathrm{sc}$, but this is not needed immediately.
Moreover, for $v\in V_\mathrm{sc}\setminus S$, we have $P_{G_\mathrm{sc}}(w,v)=1/2$ for all $(w,v)\in E_\mathrm{sc}$ appearing in \eqref{station}, and hence
\begin{equation} \label{slit}
    \phi(v) = \sum_{w: (w,v)\in E_\mathrm{sc}} \phi(w)/2.
\end{equation}
Applying \eqref{nonslit} to $T_0$, and noting that there is a single outgoing edge $(\mathtt{r},T_0)$, we obtain
\begin{equation} \label{origin}
    \phi(T_0) = \phi(\mathtt{r}).
\end{equation}
Using \eqref{origin}, we invoke \eqref{slit} to $v\in V_\mathrm{sc}\setminus S=\{T_1,\dots,T_{\nu^{\h}-1},v_1,\dots,v_{\nu^{\h}-1}\}$, which yields
\begin{equation} \label{inductive1}
    \phi(T_i) = \phi(v_i) = \phi(T_{i-1})/2= \phi(\mathtt{r})/2^i \quad\text{ for }\quad i=1,2,\dots,\nu^{\h}-1.
\end{equation}
At the same time, applying \eqref{nonslit} to $v_{\nu^{\h}}$, and using that $(T_{\nu^{\h}-1},v_{\nu^{\h}})$ is the unique outgoing edge, leads to
\begin{equation} \label{inductive2}
    \phi(v_{\nu^{\h}}) = \phi(\mathtt{r})/2^{\nu^{\h}-1}.
\end{equation}
From \eqref{nonslit}, \eqref{inductive1}, \eqref{inductive2}, and the fact that $\nu\geq 2$, it follows that for every node $v$ at computation layer $l=1$, we have
\begin{equation} \label{reduce}
    \frac{\phi(\mathtt{r})}{2^{\nu^{\h}-2}} \leq \frac{2^{\nu-1}}{2^{\nu^{\h}-1}} \cdot \phi(\mathtt{r}) =\sum_{i=0}^{\nu-2} \frac{\phi(\mathtt{r})}{2^{\nu^{\h}-1-i}} + \frac{\phi(\mathtt{r})}{2^{\nu^{\h}-1}} \leq \phi(v) \leq \sum_{i=1}^{\nu} \frac{\phi(\mathtt{r})}{2^i} = \Big(1- \frac{1}{2^{\nu}}\Big) \cdot \phi(\mathtt{r}).
\end{equation}
Proceeding inductively through the higher layers, and noting \eqref{trivial}, we gather for $v$ in computation layer $l=1,2,\dots,\h-1$, 
\begin{equation} \label{sum2}
    \frac{\nu^{l-1}}{2^{\nu^{\h}-2}} \cdot \phi(\mathtt{r}) \leq \phi(v) \leq \min\Big\{\Big(1- \frac{1}{2^{\nu}}\Big) \cdot \nu^{l-1}, 1\Big\} \cdot \phi(\mathtt{r}).
\end{equation}
In view of \eqref{inductive1}, \eqref{inductive2}, \eqref{reduce}, \eqref{sum2}, the smallest values of $\phi(v)$ occur when $v=T_{\nu^{\h}-1}, v_{\nu^{\h}-1}, v_{\nu^{\h}}$. 
Next, we derive estimates for several key entries of $Q_{G_\mathrm{sc}}$ \eqref{QG}.
We observe the following.
Starting from $\mathtt{r}$, a probability mass is injected into the graph. The mass first traverses the edge $(\mathtt{r},T_0)$. At each tape position $T_i$, for $i=0,1,\dots,\nu^{\h}-2$, the outgoing mass is split uniformly between two outgoing edges $(T_i,T_{i+1})$ and $(T_i,v_{i+1})$.
At the tape position $T_{\nu^{\h}-1}$, the outgoing mass travels along the one outgoing edge $(T_{\nu^{\h}-1},v_{\nu^{\h}})$ in $G_\mathrm{sc}$.
Consequently, the mass reaching $T_i$ equals $1/2^i$, and the mass exiting to the base node $v_{i+1}$ equals $1/2^{i+1}$, for $i=0,1,\dots,\nu^{\h}-2$, while the base node $v_{\nu^{\h}}$ receives mass $1/2^{\nu^{\h}-1}$. Equivalently,
\begin{equation} \label{masstravel}
    \begin{alignedat}{4}
        &\mathbb{P}(\mathtt{r}\to T_i) &&= 1/2^i \quad\text{ for }\quad i=1,2,\dots,\nu^{\h}-1, \quad\text{ and }\quad &&\mathbb{P}(\mathtt{r}\to T_0) &&= 1,\\
        &\mathbb{P}(\mathtt{r}\to v_i) &&= 1/2^i \quad\text{ for }\quad i=1,2,\dots,\nu^{\h}-1,
        \quad\text{ and }\quad &&\mathbb{P}(\mathtt{r}\to v_{\nu^{\h}}) &&= 1/2^{\nu^{\h}-1}.
    \end{alignedat}
\end{equation}
From each base node, the mass is then forwarded deterministically along its single outgoing edge to its single parent node and eventually reaches $\mathtt{r}$.
Now every excursion starting from a base node $v_i$ must pass through $\mathtt{r}$, and a return to $v_j$ \emph{can occur only} after exiting $\mathtt{r}$.
Thus, the event of hitting $v_j$ before returning to $v_i$ for the first time after leaving $v_i$ is equivalent to the event that an excursion initiated at $\mathtt{r}$ reaches $v_j$ before $v_i$. 
In particular, 
\begin{equation} \label{basetravel}
    \begin{alignedat}{5}
        &Q_{G_\mathrm{sc}}(v_i,v_j) &&= \mathbb{P}[\tau_{v_j} &&< \tau_{v_i}|X_0 = v_i] &&= \mathbb{P}(\mathtt{r}\to v_j) &&= 1/2^j \quad\text{ for }\quad j=1,2,\dots,\nu^{\h}-1, \\
        &Q_{G_\mathrm{sc}}(v_i,v_{\nu^{\h}}) &&= \mathbb{P}[\tau_{v_{\nu^{\h}}} &&< \tau_{v_i}|X_0 = v_i] &&= \mathbb{P}(\mathtt{r}\to v_{\nu^{\h}}) &&= 1/2^{\nu^{\h}-1}.
    \end{alignedat}
\end{equation}
Similarly, starting at $T_i$, the probability of hitting $T_j$ before returning to $T_i$, for $j<i$, is given by
\begin{align} \label{tapetravel}
    \nonumber 
    Q_{G_\mathrm{sc}}(T_i,T_j) &= \mathbb{P}\Big[\mathtt{r}\to T_j \Big| \bigcup_{l=i+1}^{\nu^{\h}} (T_i \to v_l) \text{ and } (v_l\to \mathtt{r})\Big] \mathbb{P}\Big[\bigcup_{l=i+1}^{\nu^{\h}} (T_i \to v_l) \text{ and } (v_l\to \mathtt{r})\Big] \\
    &= \mathbb{P}(\mathtt{r}\to T_j);
\end{align}
note that the complement event $\{\bigcap_{l=i+1}^{\nu^{\h}} (T_i \not\to v_l) \text{ or } (v_l \not\to \mathtt{r})\}$ cannot happen. 
All of the preceding discussion identifies two relevant candidates for the graph diameter of $G_\mathrm{sc}$: the distance between $v_{\nu^{\h}-1}$ and $v_{\nu^{\h}}$, and the distance between $T_{\nu^{\h}-2}$ and $T_{\nu^{\h}-1}$. All remaining pairs either involve larger values of $\phi$ or admit strictly larger probabilities $Q_\mathrm{sc}$, and therefore yield smaller distances.
For the first candidate distance, it follows from \eqref{boydspecial}, \eqref{inductive1}, \eqref{basetravel} that
\begin{equation} \label{first}
    E_{G_\mathrm{sc}}(v_{\nu^{\h}-1},v_{\nu^{\h}}) = \phi(v_{\nu^{\h}-1})Q_{G_\mathrm{sc}}(v_{\nu^{\h}-1},v_{\nu^{\h}}) = \phi(\mathtt{r})/2^{2(\nu^{\h}-1)}.
\end{equation}
For the second candidate, using \eqref{boydspecial}, \eqref{inductive1}, \eqref{masstravel}, \eqref{tapetravel}, we obtain
\begin{equation} \label{second}
    E_{G_\mathrm{sc}}(T_{\nu^{\h}-1},T_{\nu^{\h}-2}) = \phi(T_{\nu^{\h}-1})Q_{G_\mathrm{sc}}(T_{\nu^{\h}-1},T_{\nu^{\h}-2}) = \phi(\mathtt{r})/2^{2\nu^{\h}-3}.
\end{equation}
To calculate $\phi(\mathtt{r})$, we apply \eqref{nonslit}, \eqref{origin}, \eqref{inductive1}, and the normalization \eqref{Perronsum} to get
\begin{alignat*}{2} 
    1 = \sum_{v\in V_\mathrm{sc}} \phi(v) &= \phi(\mathtt{r})\cdot \Big(2+\sum_{i=1}^{\nu^{\h}-1} \frac{1}{2^i} \Big) + \sum_{l=0}^{\h-1} \sum_{v: v\text{ in layer } l} \phi(v) && \\
    &= \phi(\mathtt{r})\cdot \Big(3 - \frac{1}{2^{\nu^{\h}-1}} \Big) + \sum_{l=0}^{\h-1} \sum_{v: v\text{ in layer } l+1} \phi(v) 
    &&= \phi(\mathtt{r})\cdot \Big(3 - \frac{1}{2^{\nu^{\h}-1}} \Big) + \h\phi(\mathtt{r})\\
    & &&= \phi(\mathtt{r})\cdot \Big(3 - \frac{1}{2^{\nu^{\h}-1}} + \h \Big).
\end{alignat*}
Consequently, $\h^{-1} \lesssim \phi(\mathtt{r}) \leq 1/3$. 
Therefore, by definition \eqref{hittingmetric} and \eqref{first}, \eqref{second}, we conclude
\begin{align*}
    \mathrm{diam}(G_\mathrm{sc}) &= \max\{d_{G_\mathrm{sc}}(v_{\nu^{\h}-1},v_{\nu^{\h}}), d_{G_\mathrm{sc}}(T_{\nu^{\h}-1},T_{\nu^{\h}-2}) \}\\
    &= \max\big\{\log\big(E_{G_\mathrm{sc}}(v_{\nu^{\h}-1},v_{\nu^{\h}})^{-1}\big), \log\big(E_{G_\mathrm{sc}}(T_{\nu^{\h}-1},T_{\nu^{\h}-2})^{-1}\big) \big\}\lesssim \nu^{\h} + \log\h \lesssim \nu^{\h},
\end{align*}
which is \eqref{prehand}.
\end{proof}

\begin{remark} \label{rem:trudim}
When restricted to $\Gamma$, $\mathrm{diam}(\Gamma)$ is realized among pairs of computation nodes at layer $l=1$. For any two such nodes $v,w$, 
\begin{equation*}
    E_{G_\mathrm{sc}}(v,w) = \phi(v)\mathbb{P}\Big(\bigcup_{w': w' \text{ child of } w} (\mathtt{r}\to w')\Big).
\end{equation*}
Let $v$ be the parent of $v_{\nu^{\h}-\nu+1}, \dots, v_{\nu^{\h}}$ and $w$ the parent of $v_{\nu^{\h}-2\nu+1}, \dots, v_{\nu^{\h}-\nu}$ in the base layer.
Observing that, from \eqref{masstravel}
\begin{equation*}
    \mathbb{P}\Big(\bigcup_{w': w' \text{ child of } w} (\mathtt{r}\to w')\Big) \leq \mathbb{P}(\mathtt{r}\to T_{\nu^{\h}-2\nu}) = 1/2^{\nu^{\h}-2\nu}
\end{equation*}
and that the second inequality in \eqref{reduce} is attainable at $v$, we derive $E_{G_\mathrm{sc}}(v,w)\in\mathcal{O}(2^{-2\nu^{\h}+3\nu})$. Therefore, $\mathrm{diam}(\Gamma)\gtrsim \nu^{\h}$.
Thus, from this perspective, passing to the larger graph $G_\mathrm{sc}$ does not change the order of the diameter.
\end{remark}

\begin{proof}[Proof of Proposition~\ref{prop:est}(ii)]
By definition \eqref{boydspecial}, \eqref{hittingmetric}, for any $v\not=w\in\Gamma$, 
\begin{equation} \label{premin}
    \min_{v\not=w\in\Gamma} d_{G_\mathrm{sc}}(v,w) = \min_{v\not=w\in\Gamma} \log\Big(\frac{1}{\phi(v)Q_{G_\mathrm{sc}}(v,w)}\Big).
\end{equation}
Here, the value $Q_{G_\mathrm{sc}}(v,w)$ can be as large as $1$, which occurs when $w$ is the sole parent of $v$ in the computation tree. 
It then follows from \eqref{trivial}, \eqref{sum2} and the noted boundedness of $\phi(\mathtt{r})$ that $\phi(v)Q_{G_\mathrm{sc}}(v,w) \leq 1/3$. 
Combined with \eqref{premin}, this implies $\min_{v\not=w\in\Gamma} d_{G_\mathrm{sc}}(v,w) \geq 1$. 
Consequently, by \eqref{eq:boundedcomp}, for any $v\not=w\in\Gamma$,
\begin{equation*} 
    \|Q_{\eta}(v) - Q_{\eta}(w)\|_A \leq K_{\h} = \frac{K_{\h} d_{G_\mathrm{sc}}(v,w)}{d_{G_\mathrm{sc}}(v,w)} \leq \frac{K_{\h} d_{G_\mathrm{sc}}(v,w)}{\min_{v\not=w\in\Gamma}d_{G_\mathrm{sc}}(v,w)} \leq K_{\h} d_{G_\mathrm{sc}}(v,w).
\end{equation*}
Setting $\mathtt{C}_{Q_{\eta}}=K_{\h}$ gives part (ii).
\end{proof}

\begin{proof}[Proof of Proposition~\ref{prop:est}(iii)]
Recall from Definition~\ref{defn:GGCN}, $f_\mathrm{GCN}$ is given in terms of a $\mathrm{p}$-hop graph convolution, using the Laplacian $\Delta_{\Gamma}$, a $1$-Lipschitz activation function $\sigma$, and network parameters $(W_1,\dots,W_L)$ with network sizes $(\beta_1,\dots,\beta_L)$.
Recall further that the network inputs lie in $\{0,1\}^{s_{\nu,\h}}$ and the outputs take values in $\mathbb{R}^{(\m-1)\times s_{\nu,\h}}$.
Then from~\eqref{eq:GCNcompute}, 
\begin{equation} \label{beginop}
    \|f_\mathrm{GCN}(\Gamma,\cdot)\|_{\op} \leq \|\Delta_{\Gamma}\|_{\op}^{\mathrm{p} (L-1)} \cdot \prod_{l=1}^L \|W_l\|_{\op} \leq \|\Delta_{\Gamma}\|_{\op}^{\mathrm{p} (L-1)} \cdot \prod_{l=1}^L \beta_l.
\end{equation}
By definition of the operator norm
and definition~\eqref{digraphLap}, we deduce (see also Remark~\ref{rem:restrict} below)
\begin{equation} \label{deltaop}
    \|\Delta_{\Gamma}\|_{\op} \leq \|\Delta_{G_\mathrm{sc}}\|_{\op} \leq \|\Phi_{G_\mathrm{sc}}\|_{\op}\cdot \Big( 1+ \frac{\|P_{G_\mathrm{sc}}\|_{\op}+\|P_{G_\mathrm{sc}}^{\top}\|_{\op}}{2} \Big).
\end{equation}
By the row-stochasticity of $P_{G_\mathrm{sc}}$ \eqref{transition}, we have $\|P_{G_\mathrm{sc}}\|_{\op}=1$. However, $\|P_{G_\mathrm{sc}}^{\top}\|_{\op}>1$; particularly,
\begin{equation} \label{unbalance}
    \|P_{G_\mathrm{sc}}^{\top}\|_{\op} = \max_{v\in V_\mathrm{sc}} \sum_{w: (w,v)\in E_\mathrm{sc}} \frac{1}{D_{G_\mathrm{sc}}(w,w)} = \nu,
\end{equation}
where the maximum is achieved at any computation node. 
From definition and \eqref{Perronsum}, we have that $\|\Phi_{G_\mathrm{sc}}\|_{\op}\leq 1$.
Substituting these estimates, together with \eqref{deltaop}, \eqref{unbalance}, back into \eqref{beginop}, yields
\begin{equation} \label{finalop}
     \|f_\mathrm{GCN}(\Gamma,\cdot)\|_{\op} \leq \Big(\frac{3+\nu}{2}\Big)^{\mathrm{p} (L-1)} \cdot \prod_{l=1}^L \beta_l,
\end{equation}
which implies, for any $\mathbf{x}_1, \mathbf{x}_2\in\{0,1\}^{s_{\nu,\h}}$,
\begin{equation} \label{finalop1}
    \|f_\mathrm{GCN}(\Gamma,\mathbf{x}_1) - f_\mathrm{GCN}(\Gamma,\mathbf{x}_2)\|_{\infty} \leq \Big(\frac{3+\nu}{2}\Big)^{\mathrm{p} (L-1)} \cdot \prod_{l=1}^L \beta_l \|\mathbf{x}_1 - \mathbf{x}_2\|_{\infty}.
\end{equation}
By Definition~\ref{defn:GGCN}, $\pi_v\circ f_\mathrm{GCN}(\Gamma,\mathbf{0})=\pi_w\circ f_\mathrm{GCN}(\Gamma,\mathbf{0})=0$.
Thus, for $v\not=w\in\Gamma$,
\begin{multline*}
    \|\pi_v\circ f_\mathrm{GCN}(\Gamma,\mathbf{x}) - \pi_w\circ f_\mathrm{GCN}(\Gamma,\mathbf{x})\|_{\infty} \\
    \leq \|\pi_v\circ f_\mathrm{GCN}(\Gamma,\mathbf{x}) - \pi_v\circ f_\mathrm{GCN}(\Gamma,\mathbf{0})\|_{\infty} + \|\pi_w\circ f_\mathrm{GCN}(\Gamma,\mathbf{0}) - \pi_w\circ f_\mathrm{GCN}(\Gamma,\mathbf{x})\|_{\infty}.
\end{multline*}
Therefore, from \eqref{finalop1} and the fact that $\min_{v\not=w\in\Gamma}d_{G_\mathrm{sc}}(v,w)\geq 1$,
\begin{align} \label{finalop2}
    \nonumber \|\pi_v\circ f_\mathrm{GCN}(\Gamma,\mathbf{x}) - \pi_w\circ f_\mathrm{GCN}(\Gamma,\mathbf{x})\|_{\infty} &\leq 2\Big(\frac{3+\nu}{2}\Big)^{\mathrm{p} (L-1)} \cdot \prod_{l=1}^L \beta_l \cdot \frac{d_{G_\mathrm{sc}}(v,w)}{\min_{v\not=w\in\Gamma}d_{G_\mathrm{sc}}(v,w)} \\
    &\leq 2\Big(\frac{3+\nu}{2}\Big)^{\mathrm{p} (L-1)} \cdot \prod_{l=1}^L \beta_l \cdot d_{G_\mathrm{sc}}(v,w).
\end{align}
Recalling that $\mathrm{ilr}^{-1}: (\mathbb{R}^{\m-1},\|\cdot\|_2) \to (\Delta^{\circ}_{\m},\|\cdot\|_A)$ is an isometry, \eqref{finalop2} then gives, for $v\not= w\in\Gamma$ and $h\in\mathcal{H}$,
\begin{align} \label{LipGCN}
    \nonumber \|\pi_v\circ h(\mathbf{x}) - \pi_w\circ h(\mathbf{x})\|_A &\leq (\m-1)^{1/2} \|\pi_v\circ f_\mathrm{GCN}(\Gamma,\mathbf{x}) - \pi_w\circ f_\mathrm{GCN}(\Gamma,\mathbf{x})\|_{\infty} \\
    &\leq 2(\mathtt{m}-1)^{1/2}\Big(\frac{3+\nu}{2}\Big)^{\mathrm{p} (L-1)} \cdot \prod_{l=1}^L \beta_l \cdot d_{G_\mathrm{sc}}(v,w).
\end{align}
Thus, taking $\mathtt{C}_{\mathcal{H}}$ to be the constant multiplying $d_{G_\mathrm{sc}}(v,w)$ in \eqref{LipGCN} proves part (iii).
\end{proof}

\begin{remark} \label{rem:restrict}
We estimate $\|\Delta_{\Gamma}\|_{\op}$ by an upper bound of $\|\Delta_{G_\mathrm{sc}}\|_{\op}$. However, this entails no substantive loss as $\|P_{G_\mathrm{sc}}\|_{\op}, \|\Phi_{G_\mathrm{sc}}\|_{\op}\leq 1$, and the maximal diagonal entry of $\Phi_{G_\mathrm{sc}}$ can be verified to be $\phi(\mathtt{r})$. 
Moreover, by \eqref{unbalance}, the maximum determining the value of $\|P_{G_\mathrm{sc}}^{\top}\|_{\op}$ is attained at a computation node in $\Gamma$.
\end{remark}

With parts (i), (ii), (iii) established, this completes the proof of Proposition~\ref{prop:est}. 
\end{proof}


\end{document}